\documentclass[accepted]{uai2022} 

\usepackage[american]{babel}

\usepackage{natbib} 
\usepackage{mathtools} 
\usepackage{booktabs} 
\usepackage{tikz} 

\usepackage{times}  
\usepackage{helvet}  
\usepackage{courier}  
\usepackage{caption} 
\usepackage{algorithm}

\usepackage[utf8]{inputenc} 
\usepackage[T1]{fontenc}    
\usepackage{hyperref}       
\usepackage{url}            
\usepackage{booktabs}       
\usepackage{amsfonts}       
\usepackage{nicefrac}       
\usepackage{microtype}      
\usepackage{xcolor}         

\usepackage{multirow}
\usepackage{graphicx}
\usepackage{pifont}
\usepackage[utf8]{inputenc} 
\usepackage[T1]{fontenc}    
\usepackage{url}            
\usepackage{booktabs}       
\usepackage{amsfonts}       
\usepackage{nicefrac}       
\usepackage{microtype}      
\usepackage{hyperref}
\usepackage{amsfonts}
\usepackage{amssymb,color}
\usepackage{bbm}
\usepackage{mathrsfs}
\usepackage{amsmath}
\usepackage{amssymb,amsthm}

\usepackage{amsmath,amsfonts,bm}

\def\eqref#1{(\ref{#1})}

\def\1{\bm{1}}

\DeclareMathAlphabet{\mathsfit}{\encodingdefault}{\sfdefault}{m}{sl}
\SetMathAlphabet{\mathsfit}{bold}{\encodingdefault}{\sfdefault}{bx}{n}

\DeclareMathOperator*{\argmax}{arg\,max}

\usepackage{adjustbox}
\usepackage{lipsum}
\usepackage{wrapfig}
\usepackage{booktabs}
\usepackage{multirow,mathtools} 
\usepackage{algorithm,algpseudocode}
\usepackage{threeparttable}

 \algnewcommand{\algorithmicforeach}{\textbf{for each}}
\algdef{SE}[FOR]{ForEach}{EndForEach}[1]
  {\algorithmicforeach\ #1\ \algorithmicdo}
  {\algorithmicend\ \algorithmicforeach}

\usepackage{times}

\DeclareMathOperator*{\minimize}{\text{minimize}}
\DeclareMathOperator*{\maximize}{\text{maximize}}

\DeclareMathAlphabet\mathbfcal{OMS}{cmsy}{b}{n}
\newcommand{\Def}[0]{\mathrel{\mathop:}=}
\newcommand{\Deff}[0]{=\mathrel{\mathop:}}

 \newcommand{\SL}[1]{\textcolor{red}{SL: #1}}
  
 \newcommand{\revision}[1]{\textcolor{blue}{#1}}

\def\remark{\addtocounter{remark}{1}\def\@currentlabel{\theremark}%
\emph{Remark~\theremark}. } \makeatother

\newcounter{remark}

\def\b1{{\boldsymbol 1}}

\def\btheta{\boldsymbol{\theta}}
\def\bdelta{\boldsymbol{\delta}}

\def\bdelta{\boldsymbol{\delta}}

\def\bgo{\mathbf{g}_t}

\def\bg{\hat {\mathbf g}_t}

\def\bx{\mathbf x}

\def\bm{\mathbf m}

\def\by{\mathbf y}
\def\bu{\mathbf u}
\def\bv{\mathbf v}

\def\maximize{\mathop{\text{maximize}}}
\def\minimize{\mathop{\text{minimize}}}

\def\leref#1{Lemma~\ref{#1}}

\def\thref#1{Theorem~\ref{#1}}

\newtheorem{lemma}{Lemma}

\newtheorem{theorem}{Theorem}
\newtheorem{corollary}{Corollary}
\newtheorem{definition}{Definition}

\usepackage{color, colortbl}
\definecolor{Gray}{gray}{0.9}
\definecolor{Orange}{rgb}{1,0.5,0}

\makeatletter
\newcommand*{\rom}[1]{\expandafter\@slowromancap\romannumeral #1@}
\makeatother
\newcommand{\mycomment}[1]{}

\newcommand{\layernum}{h}
\newcommand{\convacc}{\xi}
\newcommand{\layerscale}{\tau}
\newcommand{\modeldim}{d}
\newcommand{\increase}[1]{\textcolor{red}{{#1}}}
\newcommand{\decrease}[1]{\textcolor{blue}{{#1}}}
\newcommand{\high}[1]{\textcolor{purple}{{#1}}}

\title{Distributed Adversarial Training to Robustify Deep Neural Networks at Scale}

\author[1,*]{Gaoyuan Zhang}
\author[2,*]{Songtao Lu}
\author[3]{Yihua Zhang}
\author[4]{Xiangyi Chen}
\author[2]{Pin-Yu Chen}
\author[1]{Quanfu Fan}
\author[1]{Lee Martie}
\author[2]{Lior Horesh}
\author[4]{Mingyi Hong}
\author[1,3]{Sijia Liu}
\affil[1]{%
   MIT-IBM Watson AI Lab, IBM Research 
}
\affil[2]{%
     Thomas J. Watson Research Center, IBM Research
}
\affil[3]{%
    Michigan State University
}
\affil[4]{%
    University of Minnesota
  }
  
\affil[*]{%
    Equal Contribution
  }
  
\begin{document}
\maketitle

\begin{abstract}
{Current} deep neural networks (DNNs) are vulnerable to adversarial attacks, where adversarial perturbations to the inputs can change or manipulate  classification. 
 To defend against such attacks, an effective and popular approach, known as \textit{adversarial training (AT)}, has been shown to mitigate the {negative} impact of adversarial attacks by virtue of a min-max robust training method. 
 While effective, it remains unclear whether it can successfully be adapted to the distributed learning context. 
 The power of distributed optimization over multiple machines 
  enables us to scale up robust training over large models and  datasets. Spurred by that, 
 we propose
 \textit{distributed adversarial training ({DAT})},
 a \textit{large-batch} adversarial training framework implemented over multiple machines. We show that {DAT} is general, which supports training over labeled and unlabeled data,
multiple types of attack generation methods, and   gradient compression operations favored for distributed optimization.
 Theoretically, we provide, under standard conditions in the optimization theory, the convergence rate of {DAT} to the first-order stationary points in general non-convex settings. Empirically, we demonstrate that {DAT} either matches or outperforms state-of-the-art robust accuracies and achieves a graceful training
 speedup (e.g., on ResNet--50 under ImageNet). Codes are available at \url{https://github.com/dat-2022/dat}.
\end{abstract}

\section{Introduction}

The rapid increase of research in DNNs and their adoption in practice is, in part, owed to the significant breakthroughs made with DNNs in computer vision \citep{alom2018history}.  
Yet, with the apparent power of DNNs, there remains a serious weakness of robustness. That is, DNNs can easily be manipulated (by an adversary) to output drastically different classifications and can be done so in a controlled and directed way. 
This process is known as an adversarial attack and considered as one of the  major hurdles in using DNNs in security critical and real-world applications
 \citep{Goodfellow2015explaining,szegedy2013intriguing,carlini2017towards,papernot2016cleverhans,kurakin2016adversarial,eykholt2018physical,xu2019evading}.

Methods to train DNNs being robust against adversarial attacks are now a major focus in research \citep{xu2019adversarial}. But most of them are far from satisfactory \citep{athalye2018obfuscated} with the exception of the adversarial training ({AT}) approach  \citep{madry2017towards}. 
{AT} is a min-max robust training method that minimizes the worst-case training loss at adversarially perturbed examples. {AT} has inspired a wide range of state-of-the-art defenses \citep{zhang2019theoretically,sinha2018certifying,boopathy2020visual,carmon2019unlabeled,shafahi2019adversarial,zhang2019you}, which ultimately resort to min-max optimization. However,
different from standard training, AT is more computationally intensive and   is difficult to scale.

\paragraph{Motivation and challenges.} 
\textit{First}, although a `fast' version of AT (we call Fast AT) was developed in   \citep{Wong2020Fast}    where an iterative   inner maximization solver is replaced by a simplified (single-step) solution, it may suffer several problems compared to AT: unstable robust learning performance \citep{li2020towards}, over-sensitive to learning rate schedule  \citep{rice2020overfitting}, and catastrophic  forgetting of robustness against strong attacks \citep{andriushchenko2020understanding}. As a result, AT is still the dominant robust training protocol across applications. Spurred by that,  we propose DAT, a new   approach to speed up AT    by allowing for scaling batch size with distributed machines.
\textit{Second}, existing AT-type methods
 are generally built on \textit{centralized} optimization.
 The need of AT in a \textit{distributed} setting arises when centralized robust training becomes infeasible or ineffective. 
For example,
training data are distributed as they cannot centrally be stored at a single machine due to  their size or privacy. 
 Or computing units are   distributed as they
allow large-batch optimization to improve the   scalability of training.
 
  \begin{figure}[htb]
\centerline{
\hspace*{-1mm}
\begin{tabular}{c}
\hspace*{-5mm}
\includegraphics[width=.33\textwidth,height=!]{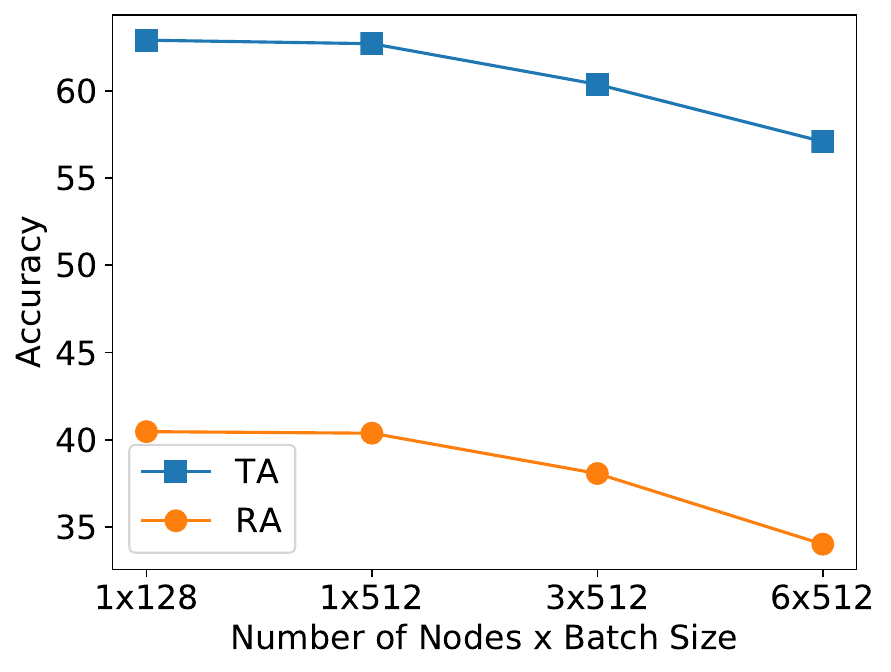} 
\end{tabular}}
\caption{\small{{{
Robust accuracy (RA) and standard test accuracy (TA) of  AT 
 vs.   scaled batch size 
 under (ImageNet, ResNet-50) using distributed machines. 
}}
}}
  \label{fig: AT_large_batch_ImageNet}
\end{figure}
While designing a distributed solution is important,  doing so effectively   is non-trivial.  Figure\,\ref{fig: AT_large_batch_ImageNet} demonstrates an example:
When scaling batch size with   the number of computing nodes,  the conventional AT method
yields a large performance drop in both robust    and standard accuracies. 
Thus, the adaptation of AT to distributed learning
leaves many unanswered questions. 
{In this work,}
we aim to design a principled and theoretically-grounded (large-batch) DAT
framework by making full use of the computing capability of multiple data-locality (distributed) machines, and show that DAT expands the capacity of data storage and the computational scalability. Furthermore, due to the existence of many variants of AT, it requires
 a  careful and systematic study on distributed AT   in its   formulation, methodology, theory and   performance evaluation.

\paragraph{Contributions.}
We list our  main contributions   below.

{\textbf{\textit{(i)}}} We provide a general algorithmic framework for DAT, which {supports}  multiple (large-batch) distributed variants of AT, e.g., supervised AT and semi-supervised AT.

    {\textbf{\textit{(ii)}}}
     In theory, we quantify how descent errors from multiple sources (gradient estimation, quantization, adaptive learning rate, and inner maximization oracle) affect the convergence of DAT. We prove that the convergence speed of DAT to the first-order stationary points in general non-convex settings at a rate of  $O({1}/{\sqrt{T}})$, where $T$ is the total number of iterations. This result matches the standard convergence rate of classic training algorithms, e.g., stochastic gradient descent (SGD), for only the minimization problems.
    
    {\textbf{\textit{(iii)}}} 
    In practice, we make a comprehensive empirical study on DAT, showing its effectiveness to (1) robust training over ImageNet,
     (2) provably robust training by  randomized smoothing, 
     (3) robust training with unlabeled data,   (4) robust pretraining + finetuning, and (5) robust training across different computing and communication configurations.

\section{Related Work}

\paragraph{Training robust classifiers.}
  AT \citep{madry2017towards}, the first known min-max optimization-based defense, has inspired a wide range of other effective defenses. Examples include adversarial logit pairing \citep{kannan2018adversarial},   input gradient or curvature regularization \citep{ross2018improving,moosavi2019robustness}, trade-off between  robustness and accuracy (TRADES) \citep{zhang2019theoretically}, distributionally robust training \citep{sinha2018certifying},
  dynamic adversarial training \citep{wang2019convergence},
  robust input attribution regularization \citep{boopathy2020visual}, certifiably robust training \citep{wong2017provable},    and semi-supervised robust training \citep{stanforth2019labels,carmon2019unlabeled}.
  
  In particular,
   some recent works  proposed   \textit{fast but approximate} AT algorithms, such as `free' AT \citep{shafahi2019adversarial},  you only propagate once (YOPO) \citep{zhang2019you}, and fast gradient sign method (FGSM) based AT  \citep{Wong2020Fast}. These algorithms achieve speedup in training  by 
simplifying the inner maximization step of AT, but are designed for centralized model training.  
{A few  works made empirical efforts to scale AT up by  using multiple computing nodes \citep{xie2019feature,kang2019testing,qin2019adversarial}, they were limited to specific use cases and  lacked a thorough study on when and how distributed learning helps,  either in theory or in practice.}

\paragraph{Distributed model training.}
Distributed optimization has been found to be effective 
  for the standard training of machine learning models \citep{dean2012large,goyal2017accurate,you2019large,chia20}. 
  In contrast to centralized optimization, distributed learning enables increasing the batch size proportional to the number of computing nodes/machines. However, it is challenging to train a model via large-batch optimization without incurring   accuracy loss compared to the standard training  with  same number of epochs \citep{krizhevsky2014one,keskar2016large}.
  To tackle this challenge, it was shown in
     \citep{you2017scaling,you2018imagenet,you2019large}  that adaptation of learning rates to   the increase of the batch size is an essential mean to boost the performance of large-batch optimization. A layer-wise adaptive learning rate strategy was then proposed  to speed up the training as well as preserve the accuracy. Although these works    have  witnessed several successful applications of distributed learning in training \textit{standard} image classifiers, 
     they leave the question of how to build \textit{robust} DNNs with {DAT} open.
     In this paper, 
     we   show that the power of layer-wise adaptive learning rate also applies to DAT.
     Since distributed learning introduces  machine-machine communication overhead,
another line of work
\citep{alistarh2017qsgd,yu2019double, bernstein2018signsgd,wangni2018gradient,stich2018sparsified,
wang2019slowmo} focused on the design of communication-efficient distributed optimization algorithms.
    
     The study on distributed learning is extensive, but the problem of distributed min-max optimization is less explored, {with some exceptions \citep{srivastava2011distributed,notarnicola2018duality,hanada2017simple,tsaknakis2020decentralized,liu2019decentralized,liu_aryan2019decentralized}}. A key difference to our work is that 
     none of the aforementioned literature studied
     the \textit{large-batch min-max optimization} with its  applications to training \textit{robust} DNNs, neither theoretically nor empirically.
     {While there are recent proposed algorithms for training Generative Adversarial Nets (GANs) \citep{liu2019decentralized,liu_aryan2019decentralized}, training robust DNNs against adversarial examples is intrinsically different from GAN training. In particular, training robust DNNs requires inner maximization with respect to each training data rather than empirical maximization with respect to model parameters.} Such an essential difference leads to different optimization goals, algorithms, convergence analyses and implementations.
     
     \mycomment{ The  work \citep{liu2019decentralized,liu_aryan2019decentralized} proposed algorithms   for training Generative Adversarial Nets (GANs).
     However, training robust DNNs against adversarial examples is intrinsically different from   GAN training, where the former requires inner maximization with respect to each training data rather than empirical maximization with respect to model parameters. }

\section{Problem Formulation}
In this section, we first review the standard  setup of    adversarial training (AT)  \citep{madry2017towards}, and then  
propose a  general    min-max    setup for distributed AT (DAT).    

\paragraph{Adversarial training.}
AT \citep{madry2017towards} is a  min-max  optimization method for training   robust ML/DL models against adversarial examples \citep{Goodfellow2015explaining}. 
Formally, AT solves the 
problem 
{
{\begin{align}
 \label{eq: adv_train}
    \begin{array}{l}
\displaystyle\minimize_{\boldsymbol{\theta}} ~\mathbb E_{(\mathbf x,  y) \in \mathcal D} \left [   \maximize_{ \| \boldsymbol{\delta} \|_\infty \leq \epsilon }  ~  \ell(\boldsymbol{\theta},  \mathbf x + \boldsymbol{\delta}; y) \right ], 
    \end{array}
\end{align}} }%
where $\boldsymbol \theta \in \mathbb R^{\modeldim}$ denotes the vector of model parameters, $\boldsymbol \delta \in \mathbb R^n$ is the vector of input perturbations within an $\ell_\infty$ ball of the given radius $\epsilon$, namely, $\| \boldsymbol{\delta} \|_\infty \leq \epsilon$, $(\mathbf x, y) \in \mathcal D$ corresponds to the training example $\mathbf x$ with label $y$ in the dataset $\mathcal D$, and $\ell$ represents a pre-defined training loss, e.g., the cross-entropy (CE) loss. The rationale behind problem \eqref{eq: adv_train} is that the model $\boldsymbol \theta$ is robustly trained against the \textit{worst-case}  loss induced by the   adversarially perturbed samples. 
It is  worth noting that the AT  problem \eqref{eq: adv_train}   is \textit{different} from   conventional stochastic min-max optimization problems, e.g.,  GANs training \citep{goodfellow2014generative}. Note that in \eqref{eq: adv_train}, the stochastic sampling corresponding to the expectation over
$(\mathbf x, \mathbf y) \in \mathcal D$  is conducted \textit{prior to}
the inner maximization operation. Such a difference leads to the \textit{sample-specific} adversarial perturbation $\boldsymbol \delta (\mathbf x) \Def \maximize_{ \| \boldsymbol{\delta} \|_\infty \leq \epsilon }  ~  \ell(\boldsymbol{\theta},  \mathbf x + \boldsymbol{\delta}; y)$.

\paragraph{Distributed AT (DAT).}
 Let us consider a popular parameter-server model of distributed learning \citep{dean2012large}.
Formally, there exist $M$ workers each of which has access to a local dataset $\mathcal D^{(i)}$, and thus 
  $\mathcal D = \cup_{i=1}^M \mathcal D^{(i)}$. There also exists a server/master node (e.g., one of workers could perform as server), which collects   local information (e.g., individual gradients)  from the other workers to  update the model parameters $\boldsymbol\theta$. 
  Spurred by  \eqref{eq: adv_train}, 
DAT   solves  problems of the following   generic form,
{
 \begin{align}\label{eq: prob_DAT}
\begin{array}{l}
\raisetag{12mm}
\displaystyle\minimize_{\boldsymbol{\theta}}  ~ \frac{1}{M}\sum_{i=1}^M  f_i(\boldsymbol \theta; \mathcal D^{(i)}), 
\\
f_i \Deff
 \mathbb E_{(\mathbf x, y) \in \mathcal D^{(i)}} \left [ \lambda \ell(\boldsymbol \theta; \mathbf x, y) + 
 \max_{\| \boldsymbol{\delta} \|_\infty \leq \epsilon} \phi(\boldsymbol \theta, \boldsymbol \delta ; \mathbf x, y)
 \right ] 
    \end{array}
\end{align}}%
where $f_i$ denotes the local cost function at the $i$th worker,  $\phi$ is a robustness regularizer against the input perturbation $\boldsymbol \delta$, and $\lambda \geq 0$ is a regularization parameter that strikes a balance between the training loss  and the worst-case robustness regularization.
In \eqref{eq: prob_DAT}, if $M = 1$, $\mathcal D^{(1)} = \mathcal D$, $\lambda = 0$ and $\phi = \ell$, then the DAT problem reduces to the  AT problem  \eqref{eq: adv_train}.
We cover two categories of  \eqref{eq: prob_DAT}.
\ding{172} DAT with {l}abeled {d}ata: 
In  \eqref{eq: prob_DAT}, we consider
$\phi(\boldsymbol \theta, \boldsymbol \delta ; \mathbf x, y)  = \ell (\boldsymbol \theta, \mathbf x+ \boldsymbol \delta ;  y) $ with labeled training data $(\mathbf x, y) \in \mathcal D^{(i)}$ for $i \in [M]$. Here $[M]$ denotes the integer set $\{ 1,2,\ldots, M\}$. 
\ding{173} DAT with {u}nlabeled {d}ata:
In \eqref{eq: prob_DAT}, 
different from DAT with labeled data,  we augment $\mathcal D^{(i)}$ with an unlabeled dataset,
and define the robust regularizer 
  $\phi$ as the pseudo-labeled worst-case CE loss \citep{carmon2019unlabeled} or the TRADES regularizer 
  \citep{stanforth2019labels,zhang2019theoretically}.
  
\mycomment{
The   robustness regularizer \eqref{eq: reg_ud} can also be specified in another form \citep{stanforth2019labels,carmon2019unlabeled}, 
$
\phi(\boldsymbol \theta, \boldsymbol \delta; \mathbf x) = \mathrm{KL}(\mathbf z(\mathbf x+\boldsymbol \delta ; \boldsymbol\theta),
        \mathbf z(\mathbf x ; \boldsymbol\theta_{\mathrm{base}})
        )
$, where $\boldsymbol\theta_{\mathrm{base}}$ denotes  a well-trained standard based model to generate  some pseudo-labels of unlabeled  samples, and $\mathrm{KL}$ represents the Kullback–Leibler divergence function.
}

\section{Methodologies}
At the first glance, distributed learning  seems being naturally applied since
 problem \eqref{eq: prob_DAT} is decomposable over multiple workers.
Yet, the actual case is much more complex.
\textbf{First}, in contrast to standard AT, DAT 
allows for using a $M$ times larger     batch size    to update the model parameters $\boldsymbol \theta$ in \eqref{eq: prob_DAT}.   Thus, given the same number of epochs, DAT takes $M$ fewer gradient updates  than AT.
Although there exist some large-batch model training techniques for  solving \textit{min-only} problems \citep{you2017large,you2017scaling,you2018imagenet,you2019large,goyal2017accurate,keskar2016large}, it remains unclear if they are effective to DAT due to its   \textit{min-max} optimization nature. 
\textbf{Second}, either AT or distributed learning has its own challenges. In AT, for ease of attack generation, i.e., conducting inner maximization of \eqref{eq: prob_DAT}, fast gradient sign method (FGSM) was leveraged to improve its computation efficiency
\citep{Wong2020Fast}. In distributed learning, 
 gradient compression   \citep{alistarh2017qsgd,yu2019double} was   used for  
reducing communication overhead. Thus, it  also
remains unclear
whether these customizations  are adaptable to DAT.  
In a nutshell, the distributed min-max optimization-based robust training algorithm has not been well studied previously, particularly in the use of different types of attack generators (inner maximization oracles), gradient quantization, large-batch size, and adaptive learning rate. Although either of the standalone techniques was studied separately, justifying their coherent integration `actually works' (both    practically and theoretically) is quite demanding.

\paragraph{Algorithmic framework of DAT.}
DAT follows the   framework of distributed learning   with   parameter server. In what follows, we elaborate on its key components
through its meta-form shown by Algorithm\,1 (see its detailed version in Algorithm\,\ref{alg: DAT}).  DAT contains three algorithmic blocks. In the \textit{first} block, 
every distributed worker    calls for a maximization oracle to obtain  the adversarial perturbation for each sample within a   data batch, then computes the gradient of the local cost function $f_i$ in \eqref{eq: prob_DAT} with respect to (w.r.t.) model parameters $\boldsymbol \theta$. And every worker is 
allowed to quantize/compress the local gradient prior to transmission to the  server.
In the \textit{second} block,
the server aggregates the local gradients, and transmits the aggregated gradient (or the quantized gradient) to the other workers.
In the \textit{third} block,
the model parameters are eventually updated by a  minimization oracle at each worker based on the received gradient information from the server.

      \begin{algorithm}[H]
          \caption{Meta-version of DAT 
        (Alg.\,\ref{alg: DAT} in Supplement)}
        \begin{algorithmic}[1]
          \For{Worker $i = 1,2, \ldots, M$} \hfill \Comment{\textcolor{blue}{Block 1}}
             \State   Sample-wise attack generation \eqref{eq: inner_max_alg}
    \State Local gradient computation \eqref{eq: stoch_grad_batch} 
\State Worker-server communication 
\EndFor
  \State Gradient aggregation at server \eqref{eq: grad_agg}  
\hfill\Comment{\textcolor{red}{Block 2}}
    \State Server-worker communication 
            \For{Worker $i = 1,2, \ldots, M$}
          \hfill \Comment{\textcolor{blue}{Block 3}}
          \State Model parameter update \eqref{eq: outer_min}
      \EndFor
        \end{algorithmic}
  \label{alg: DAT_meta_form}
      \end{algorithm}

\paragraph{Large-batch challenge in DAT and a
{l}ayerwise adaptive learning rate (LALR) solution.}
In DAT, the aggregated gradient (Step\,6 in Algorithm\,1)  
 is built on the data batch that is $M$ times larger than the standard AT. This leads to a large-batch challenge in min-max optimization.
This  challenge  can also be verified 
from Fig.\,\ref{fig: AT_large_batch_ImageNet}.
To overcome the large-batch challenge, 
we adopt the technique of 
layerwise adaptive learning rate (LALR), backed up by   the recent successful applications to the standard  training of large-scale image classification and language modeling networks with large data batch \citep{you2019large,you2017scaling}.

 To be more specific,
 the model training recipe 
 using LALR becomes
{
\begin{align}\label{eq: ada_learn}
     \boldsymbol \theta_{t+1,i} = \boldsymbol \theta_{t,i}  -    \frac{ \layerscale(\|  \boldsymbol \theta_{t,i} \|_2) \cdot \eta_t }{\| \mathbf u_{t,i} \|_2} \cdot \mathbf u_{t,i},
     \quad \forall i \in [\layernum],
 \end{align}}%
 where   \mycomment{{\color{red}Songtao: change $\btheta_{t-1}$ to $\btheta_t$ and $\btheta_t$ to $\btheta_{t+1}$ because you didn't change $\bu_t$}}
 $\boldsymbol \theta_{t,i}$ denotes the $i$th-layer  parameters at iteration $t$, with $\boldsymbol \theta_t = [ \boldsymbol \theta_{t,1}^{\top}, \ldots, \boldsymbol \theta_{t,\layernum}^{\top}  ]^{\top}$, $\layernum$ is the number of layers, $\mathbf u_t$ is a descent direction computed based on the first-order gradient w.r.t. model parameters $\btheta_t$, 
 $\layerscale(\| \boldsymbol \theta_{t,i} \|_2) = \min\{ \max \{\| \boldsymbol \theta_{t,i} \|_2, c_l\}, c_u \}$ is a \textit{layerwise} scaling factor of the \textit{adaptive} learning rate $\frac{\eta_t}{\| \mathbf u_{t,i} \|_2}$,  and  {$c_l = 0$ and $c_u = 10$ are set in our experiments (see Appendix\,\ref{app: train_setting} for some ablation studies on hyperparameter selection).} 

  In \eqref{eq: ada_learn}, the specific form of the   descent direction $\mathbf u_t$ is determined by the optimizer employed. For example, if the adaptive momentum (Adam) method is used, then $\mathbf u_t$  is given by the exponential
moving average of past gradients scaled by  square root of exponential
moving averages of squared past gradients \citep{reddi2019convergence,chen2018convergence}. Such a variant of \eqref{eq: ada_learn} that uses Adam as the base algorithm is also known as LAMB  \citep{you2019large} in standard training. 
{However, it was elusive if the advantage of LALR is preserved in large-batch min-max optimization.}
As will be evident later, the effectiveness of LALR in DAT can be justified from both theoretical and empirical perspectives.
The rationale is that  the  layer-wise adaptive learning rate   smooths the optimization trajectory so that  a larger learning rate can be used without causing sharp optima even in distributed min-max optimization.

\paragraph{Other add-ons for DAT.}
In what follows, we illustrate two add-ons to improve computation and communication efficiency of DAT.

{\ding{226}}
\textit{Inner maximization: Iterative vs. one-shot solution.}
In DAT,  each worker   calls for an inner maximization oracle to generate adversarial perturbations (Step\,2 of Algorithm\,1).  
We specify two solvers of perturbation generation: iterative projected gradient descent (PGD) and one-shot (projected) FGSM  \citep{Goodfellow2015explaining,Wong2020Fast}.
\mycomment{leading to
  the unified form 
\begin{align}
   &  \boldsymbol \delta_t^{(i)}(\mathbf x) = \mathbf z_K, \quad \mathbf z_k =  \Pi_{[-\epsilon, \epsilon]^d} [ \mathbf z_{k-1} + \alpha \cdot  \mathrm{sign}( \nabla_{\boldsymbol \delta} \phi(\boldsymbol \theta_{t}, \mathbf z_{k-1}; \mathbf x) ) ], ~ k \in [K],
   \label{eq: inner_PGD_FGSM}
\end{align}
where $K$ is the total number of iterations in the inner loop,
the cases of $K = 1$ and $K > 1$ correspond to iterative PGD attack and FGSM attack respectively,
$\mathbf z_k$ denotes the PGD update of $\boldsymbol \delta$  at the $k$th iteration,  $\mathbf z_0$ is a given intial point, 
$\Pi_{[-\epsilon, \epsilon]^d} (\cdot)$ denotes the projection    onto the box constraint  $  [-\epsilon, \epsilon]^d$, $\alpha > 0$ is a given step size, and $\mathrm{sign}(\cdot)$ denotes the element-wise sign operation.
The recent work \citep{Wong2020Fast} showed that if FGSM is conducted with   random initialization $\mathbf z_0$ and a proper step size, e.g., $\alpha = 1.25 \epsilon$, then FGSM  can be  as effective    as iterative PGD in robust training. Indeed, we will  show in Sec.\,\ref{sec: exp} that the effectiveness of our proposed DAT-FGSM algorithm echoes the finding in \citet{Wong2020Fast}.}
Our experiments will show that FGSM together with LALR works well in DAT.
We  also remark that other   techniques \citep{shafahi2019adversarial,zhang2019you}  can also be used to simplify inner maximization, however,   we   focus on FGSM since it is  computationally lightest. 

{\ding{226}} \textit{Gradient quantization.}
  In contrast to standard AT, DAT may call for   worker-server communications (Steps 4 and 7 of Algorithm\,1). That is,
  if a single-precision floating-point data type is used, then DAT needs to transmit $32\modeldim$ bits per worker-server communication at each iteration. Recall that $\modeldim$ is the dimension of   $\boldsymbol{\theta}$. In order to reduce the communication cost, DAT has the option to quantize the transmitted gradients using a 
 fixed number of bits fewer than $32$.  
 {We specify  the gradient quantization operation  
 as the   \textit{randomized quantizer}   \citep{alistarh2017qsgd,yu2019double}.
  In Sec.\,\ref{sec: exp} we will show that DAT, combined with gradient quantization, still leads to a competitive  performance. For example, the robust accuracy of ResNet-50 trained by a $8$-bit DAT
(performing quantization at Step\,4 of Algorithm\,1)
for ImageNet is just $0.55\%$  lower than the   robust accuracy achieved by the $32$-bit DAT. 
It is also worth mentioning that 
  the All-reduce communication protocol can be regarded as a special case of the parameter-server setting  in Algorithm\,1
  when   every worker performs as a server. In this case,  the communication network becomes fully connected and  only
the worker-server communication   (Step\,4 of Algorithm\,1) is needed. Please refer to Appendix\,\ref{app: alg_DAT} for more details on gradient quantization.
}

\section{Convergence Analysis  of DAT}
Although standard AT has been proved  with convergence guarantees \citep{wang2019convergence,gao2019convergence}, 
none of existing work addressed the  convergence of DAT  and   took into account LALR and gradient quantization, even in the standard AT setup.
Different from AT,
DAT needs to quantify the descent errors from multiple sources (such as gradient estimation, quantization, adaptive learning rate, and  inner maximization oracle). Before showing the challenges of proving the convergence rate guarantees, we first give the following assumptions.

\paragraph{Assumptions.}
Defining $  \Psi(\btheta) \Def \frac{1}{M}\sum^M_{i=1}
f_i(\btheta; \mathcal D^{(i)})
$ in \eqref{eq: prob_DAT}, we measure the convergence of DAT  by the first-order stationarity  of $\Psi$.
Prior to  convergence analysis, we impose the following assumptions: ($\mathcal{A}1$) $\Psi(\btheta)$ is with layer-wise  Lipschitz  continuous gradients; ($\mathcal A2$) $\phi)$ in \eqref{eq: prob_DAT} is strongly concave with respect to $\boldsymbol \delta$ and with   Lipschitz continuous gradients within the perturbation constraint; ($\mathcal A3$) Stochastic gradient   
is unbiased and has bounded variance  for each worker denoted by $\sigma^2$.
Note that the validity of ($\mathcal A2$)
could be justified from  \citep{sinha2018certifying,wang2019convergence} by imposing a strongly convex regularization into the  neighborhood of $\bdelta$. $\mathcal A2$ is   needed for tractability of analysis.  We refer readers to Appendix\,\ref{app: assumption} for more justifications on our assumptions ($\mathcal A1$)-($\mathcal A3$).  
\mycomment{In our analysis, we 
focus on the 1-sided quantization, namely, Step\,10 of Algorithm\,\ref{alg: DAT} is omitted, and
specify $\mathcal A$ by LAMB used in \citep{you2019large}.}

\paragraph{Technical challenges.}
 {In theory},  the incorporation of LALR makes the analysis of min-max optimization highly non-trivial. 
The fundamental challenge  lies in the nonlinear coupling between the biased adaptive gradient estimate resulted from LALR and the additional error generated from alternating update in DAT. From \eqref{eq: ada_learn}, we can see that the updated $\btheta$ is based on the normalized gradient, while if we perform convergence by applying the gradient Lipschitz continuity, the descent of the objective is measured by $\nabla \Psi(\btheta_t)$. This mismatch in the magnitude results in the bias term. The situation here is even worse, since the maximization problem cannot be solved exactly,   the size of the bias  depends on how close between the output of the oracle and the optimal solution w.r.t. $\bdelta$ given $\btheta$.

We have proposed
 a new descent lemma ({Lemma\,\ref{le.desc}} in Appendix)  to measure  the decrease of the objective value in the   context of alternative optimization, and showed that the bias error   resulted from the layer-wise normalization can    be compensated by  large-batch training (Theorem\,\ref{th:main_simplify}). Prior to our work, 
 we are   \textit{not} aware of any established convergence   analysis for   large-batch min-max optimization.

\paragraph{Convergence rate.}
In Theorem\,\ref{th:main_simplify}, we  present the sub-linear rate of DAT.  
\begin{theorem}\label{th:main_simplify}
Suppose that assumptions $\mathcal A1$-$\mathcal A3$ hold, the inner maximizer of DAT provides a $\varepsilon$-approximate solution 
{(i.e., the $\ell_2$-norm of inner gradient is upper bounded by $ \varepsilon$)}, 
  and the learning rate is set by   $\eta_t\sim\mathcal{O}(1/\sqrt{T})$, then $\{\btheta_t\}_{t=1}^{T}$   generated by DAT 
yields the   convergence rate
{
\begin{align}\label{eq: thr_rate}
    &\frac{1}{T}\sum^{T}_{t=1}\mathbb{E}\|\nabla_{\btheta} \Psi (\btheta_t)\|_2^2 \nonumber  \\
    = &  \mathcal O\left ( \frac{1}{\sqrt{T}} + 
    \frac{\sigma}{\sqrt{MB}} + \min\left\{ \frac{d}{4^b}, \frac{\sqrt{d}}{2^b} \right\} + \varepsilon \right ),
\end{align}}%
where $b$ denotes the number of quantization bits, and $B =\min\{|\mathcal{B}_t^{(i)}|,\forall t,i\}$ stands for the  smallest batch size per worker.
\end{theorem}
\textbf{Proof}: Please see Appendix\,\ref{app: analysis}. \hfill $\square$

The  error rate given by  \eqref{eq: thr_rate} involves four terms.  The term $\mathcal O(1/\sqrt{MB})$ characterizes the benefit of using the large per-worker  batch size $B$ and $M$ computing nodes in DAT. It is introduced since the variance of adaptive gradients (i.e., $\sigma^2$) is reduced by a factor $1/MB$, where $1/M$ corresponds to the linear speedup by $M$ machines. In \eqref{eq: thr_rate},
the term $\min\{ \frac{d}{4^b}, \frac{\sqrt{d}}{2^b} \} $ arises due to the variance of compressed gradients, 
and the other two terms imply the dependence on the number of iterations $T$ as well as the $\varepsilon$-accuracy of the inner maximization oracle.  
{We highlight that  our convergence  analysis (Theorem\,\ref{th:main_simplify}) is not barely a  combination of LALR-enabled   standard training  analysis \citep{you2019large,you2017scaling} and adversarial training  convergence analysis \citep{wang2019convergence,gao2019convergence}.
Different from the previous work, we address the
  fundamental challenges in (a) quantifying the descent property of the objective value at the presence of multi-source errors during alternating min-max optimization, and (b) deriving the theoretical relationship between large data batch (across distributed machines) and the eventual convergence error of DAT. 
}

\section{Experiments 
}\label{sec: exp}
We empirically evaluate DAT and show its success  in training robust DNNs across multiple applications, which include \ding{172} adversarially robust ImageNet training, \ding{173} provably robust training by randomized smoothing, \ding{174}  semi-supervised robust training with unlabeled data, \ding{175} robust transfer learning, \ding{176} DAT using different   communication protocols.

\subsection{Experiment setup}
\paragraph{DNN models and datasets.} 
We use  Pre-act {ResNet-18}  \citep{he2016identity}  and   {ResNet-50} \citep{he2016deep} for image classification, where the former is shortened as ResNet-18.  
And we use ImageNet \citep{deng2009imagenet}  for supervised DAT and augmented CIFAR-10 \citep{carmon2019unlabeled}  for semi-supervised DAT. 
In the latter setup, 
CIFAR-10 is augmented with
 unlabeled data drawn 
from $80$ Million Tiny Images.
When studying pre-trained model's  transferability, 
CIFAR-100 is used   as a   target dataset for down-stream classification. 

\paragraph{Computing resources.}
We train a DNN using $p$ computing nodes, each of which contains $q$  GPUs (Nvidia V100 or P100). Nodes are connected with 1Gbps ethernet.
\textit{A configuration of   computing resources is noted by $p \times q$.}
If $p > 1$, then the training is conducted in a \textit{distributed} manner.
And we split training data  into $p$ subsets, each of which is stored at a local node.  
 
We   note that the  batch size  $6 \times 512 = 3072$  is used for ImageNet over  $36$ GPUs.
Unless specified otherwise, 
DAT is conducted using Ring-AllReduce, which requires one-sided quantization in Step\,4 of Algorithm\,\ref{alg: DAT_meta_form}.

\paragraph{Baseline methods.}
We consider 2 \textit{variants} of  DAT: \underline{1)  {\textit{DAT-PGD}}}, namely, DAT using (iterative) PGD as the inner maximization oracle; and \underline{2) {\textit{DAT-FGSM}}}, namely, DAT using one-step (projected) FGSM \citep{Wong2020Fast} as the inner maximization oracle. 
Additionally, 
we consider $4$ training \textit{baselines}: \underline{1)  \textit{AT}}  \citep{madry2017towards}; \underline{2) \textit{Fast AT}} \citep{Wong2020Fast};   {\underline{3)
\textit{DAT w/o LALR}}}, namely, a   distributed implementation of AT or Fast AT
but \textit{without} considering LALR; 
and \underline{4)
\textit{DAT-LSGD}} \citep{xie2019feature}, namely, 
a {d}istributed implementation of {l}arge-batch {SGD} (LSGD) 
for  AT.
We remark that conventional AT and Fast AT are   centralized training  methods. 
\mycomment{We also find that the direct implementation of Fast-AT  in a distributed way leads to a quite poor scalability versus the growth of   batch size, and thus  a worse distributed baseline than DAT-FGSM w/o LALR. 
Lastly, we remark that the work \citep{xie2019feature} proposed modifying a model architecture  by incorporating feature denoising. In contrast, DAT does not call for architecture modification. Thus, to enable a fair comparison, we use the same training recipe LSGD as \citep{xie2019feature}  in the DAT setting, leading to the considered distributed training baseline  DAT-LSGD.
}

\begin{table}[htb]
\begin{center}
\caption{
{DAT  (in gray color) on (ImageNet, ResNet-50), compared with baselines, in TA (\%), RA (\%), AA (\%), communication time per epoch (seconds), and total training time (including communication time)  per epoch (seconds). For brevity, `$p \times q$' represents `\# nodes $\times$  \# GPUs per node', 
`C' represents communication time in seconds, and `T' represents training time in seconds. 
}
} 
\label{table: overall}
\begin{threeparttable}
\resizebox{0.48\textwidth}{!}{
\begin{tabular}{c|c|c|c|c|c|c|c}
\hline
\hline
\multirow{2}{*}{\begin{tabular}[c]{@{}c@{}}Method\end{tabular}} & \multicolumn{6}{c}{\textbf{ImageNet, ResNet-50}} \\ 
\cline{2-8}  & \begin{tabular}[c]{@{}c@{}}
$p \times q$
\end{tabular}
& Batch size
& TA (\%) & RA (\%) & AA (\%) & \begin{tabular}[c]{@{}c@{}}C (s)
\end{tabular} & \begin{tabular}[c]{@{}c@{}}
T  (s)
\end{tabular}
 \\ \hline
  AT & $1\times 6$  & $512$ & 62.70 & 40.38 & 37.46 & NA & 6022  \\
DAT-PGD w/o LALR & $6\times 6$  & $6 \times 512$ & 
57.09   
&  34.02 
& 30.98
& 865  & 1932\\
\rowcolor{Gray}
{DAT-PGD} & $6\times 6$  & $6 \times 512$ & 63.75 
& 38.45 
& 36.04
& 898  & 1960\\
{Fast AT} & $1\times 6$  & $512$ & 58.99  
& 40.78 
& 37.18 & NA & 1544  \\
{DAT-FGSM w/o LALR} & $6\times 6$  & $6 \times 512$ & 55.04 
&  35.03 
& 32.16
& 863  & 1080\\
\rowcolor{Gray}
DAT-FGSM & $6\times 6$  & $6 \times 512$ & 58.02  
& 40.27
& 36.02
& 859  & 1109\\
\hline
\hline
\end{tabular}}
\end{threeparttable}
\end{center}
\end{table}

\paragraph{Training setting.}
Unless specified otherwise, we choose the training perturbation size    $\epsilon = 8/255$ and $2/255$  for  CIFAR and ImageNet respectively, where recall that $\epsilon$ was defined in \eqref{eq: adv_train}. We also choose 
$10$ steps  and $4$ steps for PGD attack generation in DAT (and its variants) under CIFAR and ImageNet, respectively. The number of training epochs is given by $100$  for CIFAR-10 and $30$ for ImageNet. Such   training settings are consistent with previous state-of-the-art \citep{zhang2019you,Wong2020Fast}.  
To implement DAT-FGSM, we find that the use of  cyclic learning rate suggested by Fast AT \citep{Wong2020Fast} becomes over-sensitive to the increase of batch size; see Appendix\,\ref{app: add_results}. Thus, we   adopt  the standard piecewise decay step size   and an early-stop strategy \citep{rice2020overfitting} in DAT.

\paragraph{Evaluation setting.}
Unless specified otherwise, we report \underline{r}obust test \underline{a}ccuracy (\textbf{RA}) of a learned model against PGD attacks \citep{madry2017towards}.
Unless specified otherwise, we choose the  perturbation size  same as the training $\epsilon$ in evaluation, and the number of PGD steps is selected as   $20$ and $10$ for CIFAR and ImageNet, respectively. We will also measure RA against AutoAttacks and   the resulting robust accuracy is named \textbf{AA}.
Further, we   measure the standard \underline{t}est \underline{a}ccuracy (\textbf{TA}) of a model against  normal examples.
{All experiments are run   $3$ times with different random seeds, and the mean metrics are reported.}
We consider three different communication protocols, Ring-AllReduce (with one-sided quantization), parameter-server (with double quantization), and high performance computing (HPC) setting (without quantization).
To measure the communication time, we use \textsc{torch.distributed} package with gloo and nccl as communication backend\footnote{\url{https://pytorch.org/docs/stable/distributed.html}}.
We then measure the time of required worker-server communications per epoch.

\subsection{Experiment results}

\paragraph{Adversarial training on ImageNet.}
In Table\,\ref{table: overall}, we show an overall performance comparison on ImageNet between our 
proposed DAT variants and baselines 
in TA, RA, communication and computation efficiency. 
Note that AT and Fast AT are centralized training baselines using   the same number of epochs as distributed training.
\textbf{First}, we observe that the direct extension from AT (or Fast AT) to its distributed counterpart (namely, DAT-PGD w/o LALR or DAT-FGSM w/o LALR) leads to a large degradation of both RA and TA. 
  \textbf{Second},
DAT-PGD (or DAT-FGSM) is able to achieve  competitive      performance to AT (or Fast AT)
and enables a graceful training speedup.
In practice, DAT is not able to achieve linear speed-up  
mainly because of the communication cost. For example, when comparing the computation time of DAT-PGD (batch size $6 \times 512$) with that of AT (batch size $512$), the computation speed-up (by excluding the communication cost) is given by $(6022)/(1960-898) = 5.67$, consistent with the ideal computation gain using $6\times$ larger  batch size in DAT-PGD.
\textbf{Third},  when comparing DAT-FGSM with DAT-PGD, we   observe that the former leads to 
a larger loss in standard accuracy (around 5\%).
  Moreover,  similar to Fast AT, we noted  a larger RA variance  of DAT-FGSM (around 1.5\%) than DAT-PGD (around 0.5\%). Thus, the FGSM-based training is less stable than the AT-based one, consistent with  \citep{li2020towards}.

     \begin{figure}[htb]
\centerline{
\begin{tabular}{c}
\includegraphics[width=.39\textwidth,height=!]{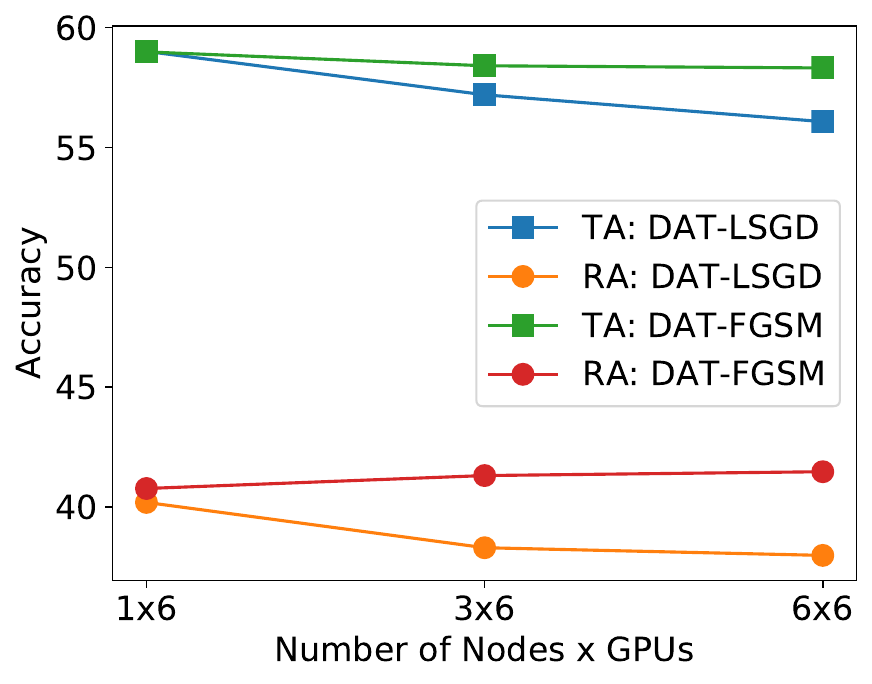}  
\end{tabular}}
\caption{\small{{{TA/RA comparison between     {DAT-FGSM} and    DAT-LSGD   vs. node-GPU configurations on (ImageNet, ResNet-50).}}
}}
  \label{fig: scalability}
\end{figure}

\begin{figure}[htb]
\centerline{
\begin{tabular}{cc}
\includegraphics[width=.23\textwidth,height=!]{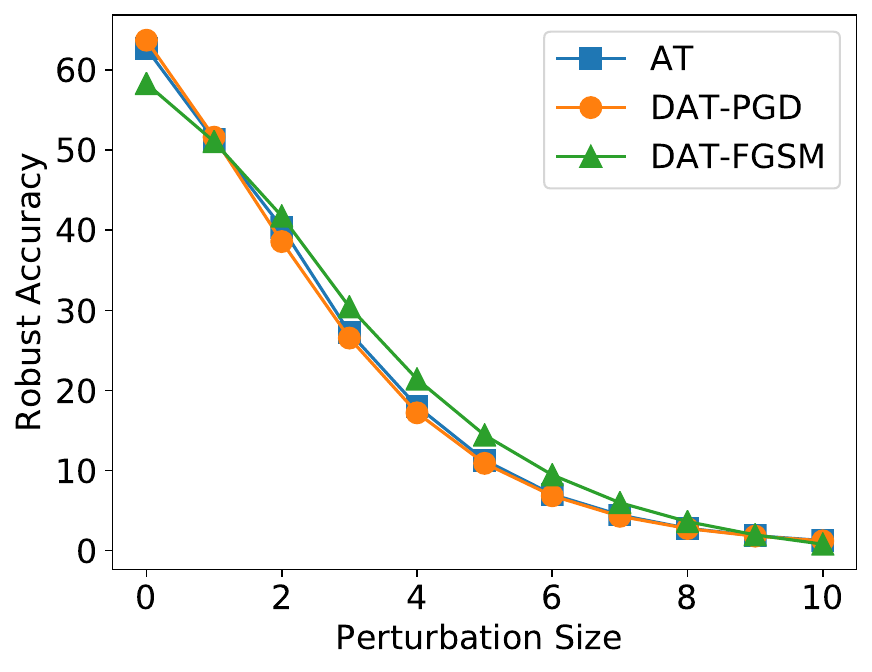} &
\includegraphics[width=.23\textwidth,height=!]{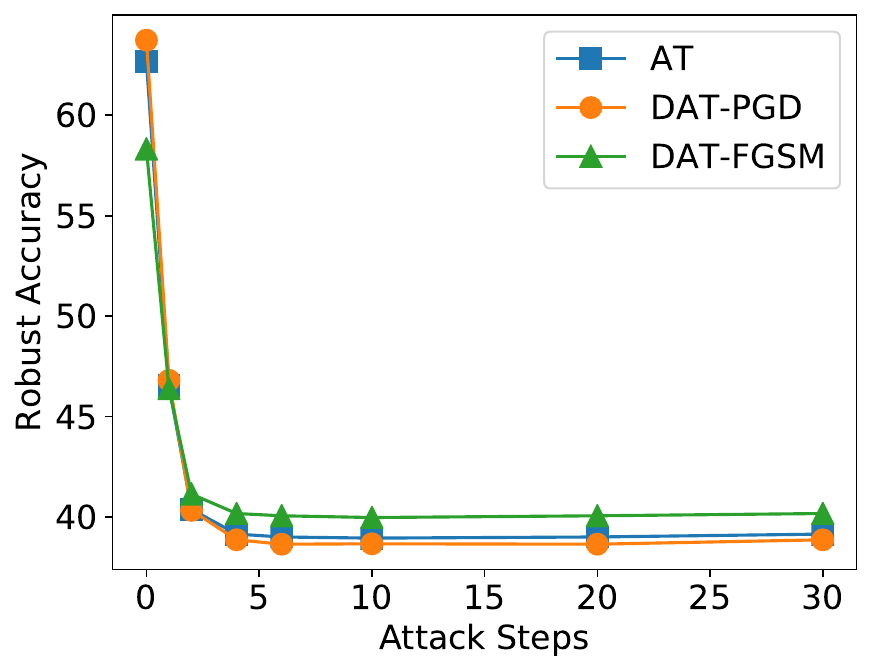}\\
 (a) & (b)
\end{tabular}}
\caption{{RA against 
PGD attacks for    model
 trained by DAT-PGD, DAT-FGSM, and AT  following (ImageNet, ResNet-50) in Table\,\ref{table: overall}.
(a) RA versus different perturbation sizes (over the divisor $255$). (b)  RA versus different steps.
}}
\label{fig: robust_PGD_imagenet}
\end{figure}

  In Figure\,\ref{fig: scalability}, we further compare our proposed DAT with the  DAT-LSGD baseline \citep{xie2019feature} in terms of TA/RA versus the number of computing nodes. Clearly, our approach scales more gracefully  than the baseline, without losing much performance as the batch sizes increases along with the number of computing nodes.
  It is worth noting that
  our DAT setup  is more challenging than \citep{xie2019feature}, which used 128 GPUs but the per-GPU utilization is the 32 batch size. By contrast, although we only use 36 GPUs, the per-GPU batch size is 85. The use of a larger batch size per GPU makes distributed robust training useful when having access to a limited computing budget.
Besides, \citet{xie2019feature} used the PGD attack generation with a quite large number of attack steps (30) at the training time. This makes the computation time dominated over the node-wise communication time. However, we  used a less number of PGD attack steps. In this scenario, the communication time cannot be neglected and prevents the practical distributed implementation from achieving the linear speed-up. For example, when comparing the computation time of DAT-PGD with that of AT in Table\,\ref{table: overall}, the computation speed-up (by excluding   communication cost) is given by $6022/(1960-898) =  5.67 $, close to the linear rate (6$\times$).

In Figure\,\ref{fig: robust_PGD_imagenet}, we evaluate the performance of DAT  against   PGD attacks of different steps
and perturbation sizes (i.e., values of $\epsilon$).
We  observe that DAT matches  robust accuracies of standard AT even against PGD attacks at different values of $\epsilon$ and steps.

\paragraph{{Adversarially
trained smooth classifier}.}
DAT also provides us an effective way to speed up the smooth adversarial training (Smooth-AT) \citep{salman2019provably}  for certified robustness. 
Different from AT, Smooth-AT augments a single training sample with  \textit{multiple}
Gaussian noisy copies so as to  train a Gaussian smoothing-aware  classifier.    Thus, Smooth-AT requires $N\times$ data batch size and storage capacity  in contrast to AT, where $N$ is the number of noisy copies per sample. In the centralized training regime, $N$ can only be set by a small value, e.g., $N = 1$ or $2$. However, DAT is able to scale up Smooth-AT  with a large value of $N$, e.g., $N = 20$ in Table\,\ref{table: RS}. 

\begin{table}[htb]
\begin{center}
\caption{{Certified accuracy (\%) of smooth classifiers   
trained by Smooth-DAT on (CIFAR-10, ResNet-18) versus   $\ell_2$ radii. Here smooth classifiers are achieved at two Gaussian noise variance levels, $\sigma = 0.12$ and $\sigma = 0.25$, following \citep{salman2019provably}. And Smooth-AT is implemented using the baseline approach with $N = 2$ and the DAT approach with $N=20$, respectively. 
}
} 
\label{table: RS}
\begin{threeparttable}
\resizebox{0.47\textwidth}{!}{
\begin{tabular}{c|c|c|c|c|c|c|c}
\hline
\hline
\multirow{2}{*}{\begin{tabular}[c]{@{}c@{}}Method\end{tabular}} & \multicolumn{6}{c}{{Smooth classifier ($\sigma = 0.12$)}} \\ 
\cline{2-8}  & $r = 0.05$
& $r = 0.1$
& $r = 0.15$ & $r = 0.2$ & $r = 0.3$ & $r = 0.4$  & $r = 0.5$
 \\ \hline
  Baseline ($N = 2$) & 0.832  & 0.804  &  0.762  & 0.728  & 0.654  & 0.545 & 0  \\
\rowcolor{Gray} DAT ($N=20$) & 0.838  & 0.812  &  0.784 & 0.748  & 0.661  & 0.550 & 0  \\
\hline
\multirow{2}{*}{\begin{tabular}[c]{@{}c@{}}Method\end{tabular}} & \multicolumn{6}{c}{{Smooth classifier ($\sigma = 0.25$)}} \\ 
\cline{2-8}  & $r = 0.05$
& $r = 0.1$
& $r = 0.15$ & $r = 0.2$ & $r = 0.3$ & $r = 0.4$  & $r = 0.5$
 \\ \hline
  Baseline ($N = 2$) & 0.752  & 0.730  &  0.708 & 0.678  & 0.625  & 0.562 & 0.498  \\
\rowcolor{Gray} DAT ($N=20$) & 0.764  & 0.748  &  0.716 & 0.688  & 0.632  & 0.566  & 0.514   \\
\hline
\hline
\end{tabular}}
\end{threeparttable}
\end{center}
\end{table}

Smooth-AT can produce a provably robust classifier \citep{cohen2019certified}.
To be more specific, let $f(\mathbf x)$ denote a classifier (with input $\mathbf x$) trained by Smooth-AT. Then its Gaussian smoothing version, given by 
$f_{\mathrm{smooth}}(\mathbf x)\Def 
\argmax_{c} \mathbb P_{\boldsymbol \delta \in \mathcal N(\mathbf 0, \sigma^2 \mathbf I)}[  f(\mathbf x+\bdelta)  = c ]
$, can achieve certified  robustness, where $c$ is a class label,   $\boldsymbol \delta \in \mathcal N(\mathbf 0, \sigma^2 \mathbf I)$  denotes the standard Gaussian noise with variance $\sigma^2$, and $\mathbb P$ signifies the majority vote-based prediction probability over multiple noisy samples. 
The resulting smooth classifier $f_{\mathrm{smooth}}$ can then be evaluated  at 
certified accuracy   \citep{cohen2019certified}, a provable robust guarantee at a given $\ell_2$ perturbation radius $r$.

In Table\,\ref{table: RS}, we present the certified accuracy (CA) of a $\sigma$-specified smooth classifier, obtained by either the conventional Smooth-AT approach (baseline using $N = 2$) or the DAT-enabled Smooth-AT method (DAT using $N = 20$). And we evaluated CA at different $\ell_2$-radii. As we can see, DAT  yields improved certified robustness over the conventional Smooth-AT with $N = 2$. This demonstrates the advantage of DAT in training provably robust classifiers: The use of a large number of Gaussian noisy samples becomes   feasible through distributed training. 
We also observe that CA drops if the perturbation $\ell_2$-radius $r$ increases. This is not surprising since CA is derived by sanity checking if the certified $\ell_2$ perturbation radius of a smooth classifier can cover a given $r$.
Besides, the smoother classifier constructed using Gaussian noises of large   variance $\sigma$ tend to be more robust against a larger $\ell_2$ perturbation radius, but may hamper the accuracy against perturbations of small $\ell_2$ radius. This is consistent with \citep{cohen2019certified} and reveals the tradeoff between accuracy and certified robustness for a $\sigma$-specific smooth classifier.

\begin{table}[htb]
\begin{center}
\caption{{DAT with semi-supervision   using  ResNet-18  or {Wide ResNet-28-10} under CIFAR-10 + 500K unlabeled Tiny Images. 
}
} 
\label{table: unlabel}
\begin{threeparttable}
\resizebox{0.45\textwidth}{!}{
\begin{tabular}{c|c|c|c|c|c}
\hline
\hline
\multirow{2}{*}{\begin{tabular}[c]{@{}c@{}}Method\end{tabular}} & \multicolumn{5}{c}{{ResNet-18, batch size $12 \times 2048$}} \\ 
\cline{2-6}  & TA (\%) & RA (\%) & AA (\%) & C(s) & T(s)\\ \hline
DAT-PGD  
& 87.00 %
& 47.34 %
& 45.23 %
& 86  & 451\\
DAT-FGSM 
& 88.00 %
& 45.84 %
& 43.19%
& 86  & 124\\
\hline
\multirow{2}{*}{\begin{tabular}[c]{@{}c@{}}Method\end{tabular}} & \multicolumn{5}{c}{{{Wide ResNet-28-10}, batch size $12 \times 128$}} \\ 
\cline{2-6}  & TA (\%) & RA (\%) & AA (\%) & C(s) & T(s)\\ \hline
DAT-PGD  
& 89.37 %
& 62.06 %
& 58.35 %
& 302  & 1020 \\
DAT-FGSM 
& 89.52 %
& 61.24 %
& 57.65 %
& 302  & 674 \\
\hline
\hline
\end{tabular}
}
\end{threeparttable}
\end{center}
\end{table}

  \begin{figure}[htb]
\centerline{
\begin{tabular}{cc}
\hspace*{-3mm} \includegraphics[width=.23\textwidth,height=!]{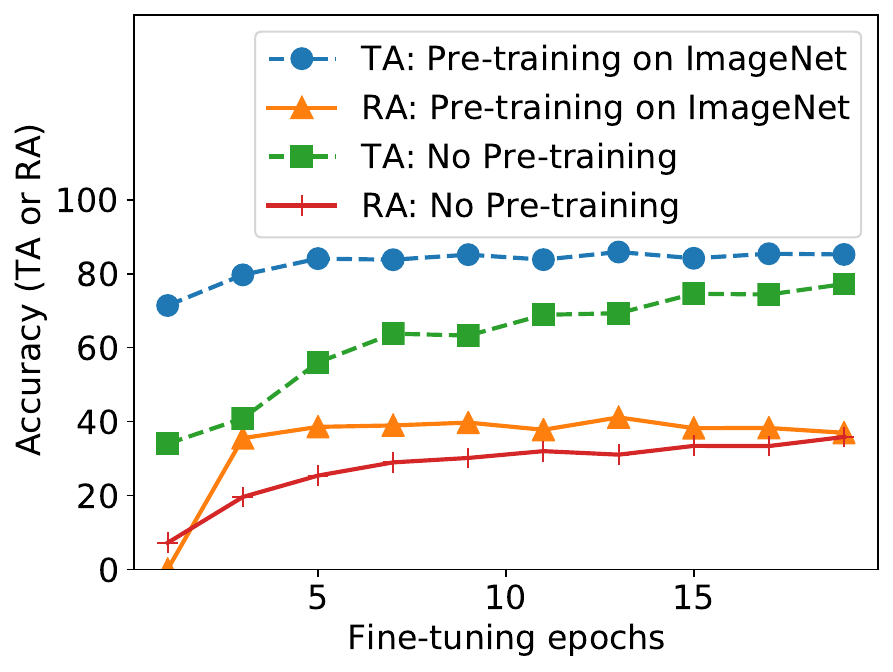} & \hspace*{-3mm}
\includegraphics[width=.23\textwidth,height=!]{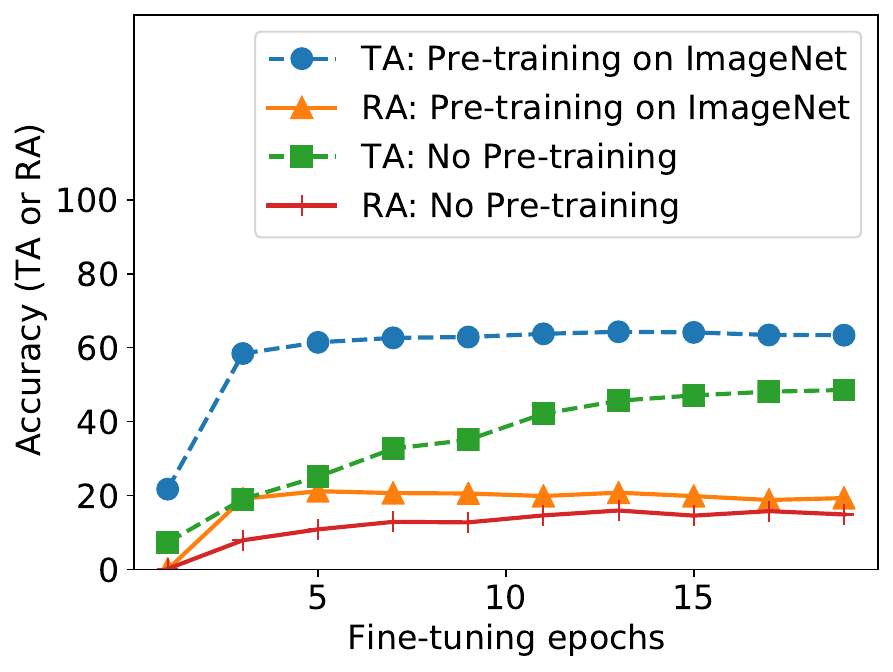}
\\
\hspace*{-3mm} \small{(a) Finetuning over CIFAR-10} & \hspace*{-3mm}
\small{(b) Finetuning over CIFAR-100}
\end{tabular}}
\caption{{Fine-tuning  ResNet-50 (pre-trained on ImageNet)  under CIFAR-10 (a) and CIFAR-100 (b). 
For compassion, adversarial training on CIFAR datasets from scratch (no pretrain) is also presented.
Here DAT-PGD is   used for both  pre-training and fine-tuning at $6$ computing nodes.
}}
  \label{fig: transfer}
\end{figure}

\paragraph{DAT under unlabeled data} 
In Table\,\ref{table: unlabel}, we report TA and RA of DAT  in the semi-supervised  setting \citep{carmon2019unlabeled} with the use of 500K  unlabeled  images mined from Tiny Images  \citep{carmon2019unlabeled}.  
As we can see, both DAT-PGD and DAT-FGSM scale well even if a $12 \times$ batch size is used across $12$ machines, each of which has a single GPU.  
We also compare  DAT-PGD with \citep{carmon2019unlabeled}  following the latter's  architecture, Wide ResNet 28-10.
We note that  DAT-PGD yields $89.37\%$ TA and $58.35\%$ AA. This is close to the reported  $89.69 \%$ TA and $59.53\%$ AA in RobustBench \citep{croce2020robustbench} built upon small-batch adversarial training.
Although the use of large data batch may cause a performance loss  due to the reduced number of training iterations, the use of data augmentation serves as a   remedy for such loss.

\paragraph{DAT from pre-training to fine-tuning.}
In Figure\,\ref{fig: transfer}, we investigate if a DAT pre-trained  model (ResNet-50) over a source dataset (ImageNet)  can offer a fast fine-tuning to a down-stream target dataset (CIFAR-10/100).
Here  we up-sample  a  CIFAR image to the same dimension of an ImageNet image before feeding it into the pre-trained model \citep{Shafahi2020Adversarially}.
Compared with the direct application of DAT to the target dataset (without pre-training), the pre-training enables a fast adaption  to the down-stream CIFAR task in both TA and RA within just  $3$ epochs. Thus, the scalability of DAT to large datasets and multiple   nodes offers a great potential to rapidly initialize an adversarially robust base model  in the  `pre-training + fine-tuning' paradigm.

\paragraph{Quantization effect in various communication protocols}
In
Table\,\ref{table: quadtization_imagenet},
we present how DAT is affected by gradient quantization. 
As we can see, when the number of bits is reduced, the communication cost and the amount of transmitted data are
saved, respectively. However, the use of an aggressive gradient quantization introduces a performance loss. 
For example, compared with
the case of using $32$ bits, 
the most aggressive quantization scheme  ($8$-bit $2$-sided quantization in Steps 4 and 7 of Algorithm\,\ref{alg: DAT_meta_form}) yields an RA drop around $4\%$
and $7\%$ for DAT-PGD and   DAT-FGSM, respectively. 
In particular,
 DAT-FGSM is more sensitive to the effect of gradient quantization than DAT-PGD. 
It is worth noting that 
our main communication configuration used in previous experiments is   Ring-AllReduce that calls for 1-sided (rather than 2-sided) quantization.
We further show that if a high performance computing (HPC) cluster of nodes (with NVLink high-speed GPU interconnect \citep{foley2017ultra}) is used, the communication cost can  be further reduced without causing performance loss.

 \begin{table}[htb]
\begin{center}
\caption{{Effect of gradient quantization on the performance of DAT  for various numbers of bits. 
The training and evaluation settings on (ImageNet, ResNet-50) are consistent with     Table\,\ref{table: overall}. The new performance metric `Data trans. (MB)' represents data transmitted per iteration in the unit MB. 
}
} 
\label{table: quadtization_imagenet}
\begin{threeparttable}
\resizebox{0.45\textwidth}{!}{
\begin{tabular}{c|c|c|c|c|c}
\hline
\hline
\mycomment{\multirow{2}{*}{\begin{tabular}[c]{@{}c@{}}Method\end{tabular}} & \multicolumn{5}{c}{\textbf{CIFAR-10, ResNet-18}
}
\\ 
\cline{2-6}  
& \# bits
& TA (\%) & RA (\%) & \begin{tabular}[c]{@{}c@{}}Comm.\\per epoch (s)\end{tabular} & \begin{tabular}[c]{@{}c@{}}Data transmitted
\\ per iteration (MB)\end{tabular}
 \\ \hline
DAT-PGD & $32$  & 
80.38 & 38.94 & 8.5  & 1278\\
DAT-PGD & $16$ 
& 
79.38 & 38.32 & 8.3  & 639\\
DAT-PGD & $8$  & 
78.18 & 37.34 & 4.3  & 320\\
DAT-PGD & $8$  ($2$-sided) & 
78.86 & 34.2 & 5.0  & 107\\
DAT-FGSM &  $32$  & 
75.58 & 40.92 & 8.5  & 1278\\
DAT-FGSM & $16$  & 
75.74 & 40.86 & 8.3 & 639  \\
DAT-FGSM &  $8$  & 
72.48 & 38.98 & 4.3 & 320  \\
DAT-FGSM & $8$ ($2$-sided) &
69.26 & 35.34 & 5.0 & 107  \\
\hline 
}
\multirow{2}{*}{\begin{tabular}[c]{@{}c@{}}Method\end{tabular}} & \multicolumn{5}{c}{\textbf{ImageNet, ResNet-50} 
} \\ 
\cline{2-6}  
& \# bits
& TA (\%) & RA (\%) & \begin{tabular}[c]{@{}c@{}}C (s)\end{tabular} & \begin{tabular}[c]{@{}c@{}}Data\\ trans.  (MB)\end{tabular}
 \\ \hline
DAT-PGD & $32$
& 63.75 & 38.45 & 898  & 2924\\
DAT-PGD & $16$ 
& 61.77 & 38.40 & 850  & 1462\\
DAT-PGD & $8$ 
& 56.53 & 37.90 & 592  & 731\\
DAT-PGD & $8$ ($2$-sided)
& 53.09 & 34.59 & 1091  & 244\\
DAT-PGD (HPC) & $32$
& 63.43 & 38.55 & 15  & 1074\\
\hline
DAT-FGSM & 32
& 58.02 & 40.27 & 859  & 2924\\
DAT-FGSM & 16 
& 54.71 & 39.29 & 849 & 1462  \\
DAT-FGSM & 8 
& 50.11 & 36.38 & 594 & 731  \\
DAT-FGSM & $8$ ($2$-sided)
& 48.27 & 33.20 & 1013 & 244  \\
DAT-FGSM (HPC) & 32
& 57.60 & 41.70 & 15  & 310\\
\hline
\hline 
\end{tabular}}
\end{threeparttable}
\end{center}
\end{table}

\section{Conclusions}
We proposed   {d}istributed {a}dversarial {t}raining (DAT)  to  scale up the training of adversarially robust DNNs over multiple machines. 
We   showed that   DAT   is general in  that it
enables large-batch min-max optimization and supports gradient compression and different learning regimes. 
We proved that under mild conditions, DAT  is guaranteed to converge to a first-order stationary point with a sub-linear rate. Empirically, we provided comprehensive experiment results to demonstrate the effectiveness and the usefulness of DAT in training robust DNNs with large datasets and multiple machines.
In the future, it will be worthwhile to examine the speedup achieved by DAT in the extreme training cases, e.g., using a significantly large number of  attack steps and distributed machines, and extremely large models. 

\begin{acknowledgements} 
Y. Zhang and S. Liu are supported by the Cisco Research   grant CG\# 70614511. P. Khanduri and M. Hong are supported in part by the NSF grants 1910385 and 1727757. We also thank Dr. Cho-Jui Hsieh for the helpful discussion on early ideas of this paper.
\end{acknowledgements}

\newpage
\clearpage
{{
\bibliography{refs}
}}

\newpage
\clearpage

\onecolumn
\setcounter{section}{0}

\section*{Supplementary Materials}

\setcounter{section}{0}
\setcounter{figure}{0}
\makeatletter 
\renewcommand{\thefigure}{A\@arabic\c@figure}
\makeatother
\setcounter{table}{0}
\renewcommand{\thetable}{A\arabic{table}}
\setcounter{mylemma}{0}
\renewcommand{\themylemma}{A\arabic{mylemma}}
\setcounter{algorithm}{0}
\renewcommand{\thealgorithm}{A\arabic{algorithm}}
\setcounter{equation}{0}
\renewcommand{\theequation}{A\arabic{equation}}

 
\section{DAT Algorithm Framework}
\label{app: alg_DAT}

\begin{minipage}{0.95\textwidth}
\centering
\begin{algorithm}[H]
\caption{Distributed adversarial training (DAT) for solving problem \eqref{eq: prob_DAT}}
\label{alg: DAT}
\begin{algorithmic}[1]
  \State Initial $\boldsymbol{\theta}_1$,  dataset $\mathcal D^{(i)}$  for each of $M$ workers,   and $T$ iterations 
\For{Iteration $t =  1,2,\ldots, T$}
 \For{Worker $i = 1,2, \ldots, M$}   \Comment{\textcolor{blue}{Worker}}
\State  Draw a finite-size data batch  $\mathcal B_{t}^{i} \subseteq \mathcal D^{(i)} $
\State  For each data sample  $\mathbf x \in \mathcal B_{t}^{i}$, call for an \textit{inner maximization oracle}:
{\small\begin{align}\label{eq: inner_max_alg}
\boldsymbol{\delta}_t^{(i)}(\mathbf x) \Def \argmax_{ \| \boldsymbol{\delta} \|_\infty \leq \epsilon }  ~  \phi(\boldsymbol{\theta}_{t}, \boldsymbol{\delta}; \mathbf  x),
\end{align}}%
\hspace*{0.4in} where we omit the label or possible pseudo-label $y$ of $\mathbf x$ for brevity 
\State Computing local gradient of $f_i$ in \eqref{eq: prob_DAT} with respect to $\boldsymbol \theta$ given perturbed samples: 
{\small\begin{align}\label{eq: stoch_grad_batch}
    \mathbf g_t^{(i)} =  \lambda \mathbb E_{\mathbf x \in \mathcal B_t^{(i)}}  [ \nabla_{\boldsymbol \theta}\ell(\boldsymbol \theta_{t}; \mathbf x)  ] + \mathbb E_{\mathbf x \in \mathcal B_t^{(i)}}  [ \nabla_{\boldsymbol \theta} \phi(\boldsymbol \theta_{t}; \mathbf x + \boldsymbol \delta_t^{(i)}(\mathbf x) )  ]
\end{align}}
\State (\textit{Optional}) Call for \textit{gradient quantizer} $Q(\cdot)$ and transmit
  $ Q(\mathbf g_t^{(i)})$ to   server
\EndFor
  \State Gradient aggregation at  server:  \hfill \Comment{\textcolor{red}{Server}}
 {\small \begin{align}\label{eq: grad_agg}
  \hat {\mathbf g}_t = \textstyle \frac{1}{M} \sum_{i=1}^M Q(\mathbf g_t^{(i)})
  \end{align}}%
\State (\textit{Optional}) Call for \textit{gradient quantizer}   $\hat{\mathbf g}_t \leftarrow Q(\hat{\mathbf g}_t) $,
  and transmit  
  $ \hat {\mathbf g}_t$ to workers: 
  
\For{Worker $i = 1,2, \ldots, M$}
 \hfill \Comment{\textcolor{blue}{Worker}}
  \State Call for an \textit{outer minimization oracle} $\mathcal A(\cdot)$ to update $\boldsymbol \theta$:  
 { \small \begin{align}\label{eq: outer_min}
      \boldsymbol \theta_{t+1} = \mathcal A(\boldsymbol \theta_{t}   \hat {\mathbf g}_t, \eta_t), \quad \quad  \text{$\eta_t$ is learning rate}
  \end{align}}%
\EndFor

  \EndFor
  \end{algorithmic}
\end{algorithm}
\end{minipage}

\paragraph{Additional details on gradient quantization}
Let $b$ denote the number of bits ($b \leq 32$), and  thus there exists $s = 2^b$  quantization levels. We specify the gradient quantization operation $Q(\cdot)$ in Algorithm\,\ref{alg: DAT} as the   \textit{randomized quantizer}   \citep{alistarh2017qsgd,yu2019double}.  Formally,
the  quantization operation at the $i$th coordinate of a vector $\mathbf g$ is given by \citep{alistarh2017qsgd}
{\small\begin{align}\label{eq: rand_q}
    Q( g_i) = \| \mathbf g \|_2 \cdot \mathrm{sign}(g_i) \cdot \xi_i(g_i,s),  \quad \forall i \in \{ 1,2, \ldots, \modeldim \}.
\end{align}}%
In \eqref{eq: rand_q},  $\xi_i(g_i,s)$ is a random number drawn as follows. Given $|g_i|/\| \mathbf g \|_2 \in [l/s, (l+1)/s]$ for some $l \in \mathbb N^+$ and $0 \leq l < s$, we  then  have
{\small\begin{align}\label{eq: xi}
\xi_i(g_i,s) = \left \{ 
    \begin{array}{ll}
      l/s   & \text{with probability $1 - (s |g_i|/\| \mathbf g \|_2 - l)$}  \\
      (l+1)/s   &  \text{with probability $ (s |g_i|/\| \mathbf g \|_2 - l)$},
    \end{array}
    \right.
\end{align}}%
where $|a|$ denotes the absolute value of a scalar $a$, and  $\| \mathbf a \|_2$ denotes the $\ell_2$ norm of a vector $\mathbf a$.
The rationale behind using \eqref{eq: rand_q} is that  $Q(g_i)$ is an \textit{unbiased} estimate of $g_i$, namely,
$
\mathbb E_{\xi_i(g_i, s)}[Q(g_i)] =  g_i
$, with bounded variance. Moreover, we at most  need
 $(32 + \modeldim + b \modeldim  )$ bits to transmit the quantized  $Q(\mathbf g)$, where $32$ bits for $\| \mathbf g \|_2$, $1$ bit for sign of $g_i$ and $b$ bits for $\xi_i(g_i,s)$,
whereas it needs $32\modeldim$ bits for a single-precision 
$\mathbf g$. Clearly, a small $b$ saves the communication cost.
We note that  if every worker performs as a server in DAT, then the quantization operation at Step\,10 of Algorithm\,\ref{alg: DAT} is no longer needed. In this case, the communication network becomes fully connected. With synchronized communication, this is favored  for   training DNNs under the All-reduce operation.

\newpage

\section{Theoretical Results}\label{app: thr_1}

In this section, we will quantify the convergence behaviour of the proposed DAT algorithm. First, we define the following notations:
\begin{equation}\label{eq.defPhi}
\Phi_i(\btheta,\bx)=\max_{\|\bdelta^{(i)}\|_{\infty}\le\epsilon}\phi(\btheta,\bdelta^{(i)};\bx),\quad \textrm{and}\quad
    \Phi_i(\btheta)=\mathbb{E}_{\bx\in\mathcal{D}^{(i)}}\Phi_i(\btheta ; \bx).
\end{equation}
We also define
\begin{equation}
    l_i(\btheta)=\mathbb{E}_{\bx\in\mathcal{D}^{(i)}} l(\btheta ; \bx),
\end{equation}
where the label $y$ of $\mathbf x$ is omitted for labeled data. 
Then, the objective function of  problem \eqref{eq: prob_DAT} can be expressed in the  compact way
\begin{equation}
    \Psi(\btheta)=\frac{1}{M}\sum^M_{i=1}\lambda l_i(\btheta)+\Phi_i(\btheta)
\end{equation}
and the optimization problem is then given by $\min_{\btheta}\Psi(\btheta)$. 

Therefore, it is clear that if a point $\btheta^{\star}$ satisfies 
\begin{equation}
    \|\nabla_{\btheta} \Psi(\btheta^{\star})\|\le\convacc,
\end{equation} then we say $\btheta^{\star}$ is a $\convacc$ approximate first-order stationary point (FOSP) of problem  \eqref{eq: prob_DAT}.

Prior to delving into the convergence analysis of DAT, we   make the following assumptions.

\subsection{Assumptions}\label{app: assumption}

A1. Assume objective function has layer-wise Lipschitz continuous gradients with constant $L_i$ for each layer
\begin{align}
\|\nabla_i \Psi(\btheta_{\cdot,i})-\nabla_i \Psi(\btheta'_{\cdot,i})\|\le L_i\|\btheta_{\cdot,i}-\btheta'_{\cdot,i}\|,\forall i\in[h].
\end{align}
where $\nabla_i\Psi(\btheta_{\cdot,i})$ denotes the gradient w.r.t. the variables at the $i$th layer. Also, we assume that $\Psi(\btheta)$ is lower bounded, i.e., $\Psi^{\star}:=\min_{\btheta} \Psi(\btheta)>-\infty$ and bounded gradient estimate, i.e., $\|\nabla \bg^{(i)}\|\le G$.

A2. Assume that $\phi(\btheta,\bdelta;\bx)$ is strongly concave with respect to $\bdelta$ with parameter $\mu$ and has the following gradient Lipschitz continuity with constant $L_{\phi}$:
\begin{equation}
   \|\nabla_{\btheta}\phi(\btheta,\bdelta;\bx)-\nabla_{\btheta}\phi(\btheta,\bdelta';\bx)\|\le L_{\phi}\|\bdelta-\bdelta'\|.
\end{equation}

A3. Assume that the gradient estimate is unbiased and has bounded variance, i.e.,
\begin{align}
    \mathbb{E}_{\bx\in\mathcal{B}^{(i)}} [\nabla_{\btheta} l(\btheta;\bx)] =& \nabla_{\btheta} l(\btheta), \forall i,
\\
    \mathbb{E}_{\bx\in\mathcal{B}^{(i)}} [\nabla_{\btheta} \Phi(\btheta;\bx)]=&\nabla_{\btheta} \Phi(\btheta),\forall i, 
\end{align}
where recall that $\mathcal{B}^{(i)}$ denotes a data batch used at worker $i$,
$\nabla_{\btheta} l(\btheta):=\frac{1}{M}\sum^M_{i=1}\nabla_{\btheta} l_i(\btheta)$ and $\nabla_{\btheta} \Phi(\btheta):=\frac{1}{M}\sum^M_{i=1}\nabla_{\btheta} \Phi_i(\btheta)$; and \begin{align}
    \mathbb{E}_{\bx\in\mathcal{B}^{(i)}}&\|\nabla_{\btheta} l(\btheta;\bx)-\nabla_{\btheta} l(\btheta)\|^2\le\sigma^2, \forall i \\ \mathbb{E}_{\bx\in\mathcal{B}^{(i)}}&\|\nabla_{\btheta} \Phi(\btheta;\bx)-\nabla_{\btheta} \Phi(\btheta)\|^2\le\sigma^2,\forall i.
\end{align}
Further, we define a component-wise bounded variance of the gradient estimate
\begin{align}
    \mathbb{E}_{\bx\in\mathcal{B}^{(i)}}&\|[\nabla_{\btheta} l(\btheta;\bx)]_{jk}-[\nabla_{\btheta} l(\btheta)]_{jk}\|^2\le\sigma^2_{jk}, \forall i,
    \\
    \mathbb{E}_{\bx\in\mathcal{B}^{(i)}}&\|[\nabla_{\btheta} \Phi(\btheta;\bx)]_{jk}-[\nabla_{\btheta} \Phi(\btheta)]_{jk}\|^2\le\sigma'^2_{jk},\forall i,
\end{align}
where  $j$ denotes the index of the layer, and $k$ denotes the index of entry at each layer. Under A3, we have $\sum^h_{j=1}\sum^{d_j}_{k=1}\max\{\sigma^2_{jk},\sigma'^2_{jk}\}\le\sigma^2$

A4. Assume that the component wise compression error has bounded variance
\begin{equation}
\mathbb{E}[(Q([\mathbf{g}^{(i)}(\btheta)]_{jk})-[\mathbf{g}^{(i)}(\btheta)]_{jk})^2]\le \delta^2_{jk},\forall i.
\end{equation}
The assumption A4 is satisfied as the randomized quantization is used \citep[Lemma\,3.1]{alistarh2017qsgd}.

\subsection{Oracle of maximization}

In practice, $\Phi_i(\btheta;\bx),\forall i$ may not be obtained, since the inner loop needs to iterate by the infinite number of iterations to achieve the exact maximum point. Therefore, we allow some numerical error term resulted in the maximization step at \eqref{eq: inner_max_alg}. This consideration makes the convergence analysis more realistic. 

First, we have the following criterion to measure the closeness of the approximate maximizer to the optimal one.

\begin{definition}
Under A2, if point $\bdelta(\bx)$ satisfies
\begin{equation}\label{eq.conde}
    \max_{\bdelta\le\|\epsilon\|}\left\langle \bdelta-\bdelta^*(\bx),\nabla_{\bdelta} \phi(\btheta,\bdelta^*(\bx);\bx)\right\rangle\le\varepsilon
\end{equation}
then, it is a $\varepsilon$ approximate solution to $ \bdelta^*(\bx)$, where
\begin{equation}\label{eq.optdelta}
    \bdelta^*(\bx):=\argmax_{\bdelta\le\|\epsilon\|}\phi(\btheta,\bdelta;\bx).
\end{equation}
and $\bx$ denotes the sampled data.

\end{definition}
Condition \eqref{eq.conde} is standard for defining approximate solutions of an optimization problem over a compact feasible set and has been widely studied in \citep{wang2019convergence,lu2019snap}.

In the following, we can show that when the inner maximization problem is solved accurately enough, the gradients of function $\phi(\btheta,\bdelta(\bx);\bx)$ at  $\bdelta(\bx)$ and $\bdelta^*(\bx)$ are also close. A similar claim of this fact has been shown in \citep[Lemma 2]{wang2019convergence}. For completeness of the analysis, we provide the specific statement for our problem here and give the detailed proof as well.
\begin{lemma}\label{le.vererror}
Let $\bdelta^{(k)}_t$ be the $(\mu\varepsilon)/L^2_{\phi}$  approximate solution of the inner maximization problem for worker $k$, i.e., $\max_{\bdelta^{(k)}}\phi(\btheta,\bdelta^{(k)};\bx_t)$, where $\bx_t$ denotes the sampled data at the $t$th iteration of DAT. Under A2, we have
\begin{equation}
\left\|\nabla_{\btheta} \phi\left(\btheta_t,\bdelta^{(k)}_t(\bx_t);\bx_t\right)-\nabla_{\btheta} \phi\left(\btheta_t,(\bdelta^*)^{(k)}_t(\bx_t);\bx_t\right)\right\|^2\le \varepsilon. \label{eq.maxorl0}
\end{equation}
\end{lemma}

Throughout the convergence analysis, we assume that $\bdelta^{(k)}_t(\bx_t),\forall k,t$ are all the $(\mu\varepsilon)/L^2_{\phi}$ approximate solutions of the inner maximization problem. Let us define
\begin{equation}
\left\|[\nabla \phi(\btheta_t,\bdelta^{(k)}_t(\bx_t);\bx_t)]_{ij}-[\nabla \phi(\btheta_t,(\bdelta^*)^{(k)}_t(\bx_t);\bx_t]_{ij}\right\|^2=\varepsilon_{ij}.
\end{equation}
From \leref{le.vererror}, we know that when $\bdelta^{(k)}_t(\bx_t)$ is a $(\mu\varepsilon)/L^2_{\phi}$  approximate solution, then 
\begin{equation}
\sum^h_{i=1}\sum^{d_i}_{j=1}\varepsilon_{ij} =\sum^h_{i=1}\sum^{d_i}_{j=1}\left\|[\nabla \phi(\btheta_t,\bdelta^{(k)}_t(\bx_t);\bx_t)]_{ij}-[\nabla \phi(\btheta_t,(\bdelta^*)^{(k)}_t(\bx_t);\bx_t]_{ij}\right\|^2\le\varepsilon. \label{eq.maxorl}
\end{equation}

\subsection{{Formal statements of} convergence rate guarantees}

{In what follows, we provide the formal statement of convergence rate of DAT.}
In our analysis, we 
focus on the 1-sided quantization, namely, Step\,10 of Algorithm\,\ref{alg: DAT} is omitted, and
specify the outer minimization oracle  by LAMB  \citep{you2019large}, see Algorithm\,\ref{alg:p1}. The  addition and multiplication operations in LAMB are component-wise. 

\begin{theorem}\label{th:main}
Under A1-A4, suppose that  $\{\btheta_t\}$ is generated by DAT for a total number of $T$ iterations, 
and let the problem dimension at each layer be $d_i=d/h$. Then the convergence rate of DAT is given by
\begin{align}
\frac{1}{T}\sum^{T}_{t=1}\mathbb{E}\|\nabla_{\btheta} \Psi(\btheta_t)\|^2\le&\frac{\Delta_{\Psi}}{\eta_t c_l CT}+2\left(\varepsilon+\frac{(1+\lambda)\sigma^2}{MB}\right)+4\delta^2 +
\frac{\kappa\sqrt{3}}{C}\|\boldsymbol{\chi}\|_1+\frac{\eta_t c_u\kappa \|L\|_1}{2C}.
\end{align}
where $\Delta_{\Psi}:=\mathbb{E}[\Psi(\btheta_1)]-\Psi^{\star}]$, 
{$\eta_t$ is the learning rate, $\kappa=c_u/c_l$, $c_l$ and $c_u$ are constants used in LALR \eqref{eq: ada_learn},}
$\boldsymbol{\chi}$ is an error term  with the $(ih+j)$th entry being $\sqrt{\frac{(1+\lambda)\sigma^2_{ij}}{MB}+\varepsilon_{ij}+\delta^2_{ij}}$,
{$\varepsilon$ and $\varepsilon_{ij}$ were given in \eqref{eq.maxorl},}
$L=[L_1,\ldots,L_h]^T$, $C=\frac{1}{4}\sqrt{\frac{h(1-\beta_2)}{G^2d}}$, $0<\beta_2<1$ is given  in LAMB, $B=\min\{|\mathcal{B}^{(i)}|,\forall i\}$, and {$G$ is given in A1.}
\end{theorem}

\remark When the batch size is large, i.e., $B\sim\sqrt{T}$, then the gradient estimate error will be $\mathcal{O}(\sigma^2/\sqrt{T})$. Further, it is worth noting that different from the convergence results of LAMB, there is a linear speedup of deceasing the gradient estimate error in DAT with respect to $M$, i.e., $\mathcal{O}(\sigma^2/(M\sqrt{T}))$, which is the advantage of using      multiple computing nodes. 

\remark Note that A4 implies $\mathbb{E}[(Q([\mathbf{g}^{(k)}(\btheta)]_{ij})-[\mathbf{g}^{(k)}(\btheta)]_{ij}\|^2]\le \sum^h_{i=1}\sum^{d_i}_{j=1}\delta^2_{ij}:=\delta^2
$. 
From \citep[Lemma 3.1]{alistarh2017qsgd}, we know that $\delta^2\le \min\{ d/s^2, \sqrt{d}/s \} G^2$. Recall that $s = 2^b$, where $b$ is the number of quantization bits.

Therefore, with a  proper choice of the parameters, we can have the following convergence result that has been shown in \thref{th:main_simplify}.

\begin{corollary}
Under the same conditions of Theorem\,\ref{th:main}, if we choose
\begin{equation}
    \eta_t\sim\mathcal{O}(1/\sqrt{T}),\quad  \varepsilon\sim\mathcal{O}(\convacc^2),
\end{equation}
 we then have 
 \begin{align}
\frac{1}{T}\sum^{T}_{t=1}\mathbb{E}\|\nabla_{\btheta} \Psi(\btheta_t)\|^2\le\frac{\Delta_{\Psi}}{c_l C\sqrt{T}}+\frac{(1+\lambda)\sigma^2}{MB}
+\frac{ c_u\kappa \|L\|_1}{2C\sqrt{T}}+\mathcal{O}\left(\convacc,\frac{\sigma}{\sqrt{MT}},\min\left\{ \frac{d}{4^b}, \frac{\sqrt{d}}{2^b} \right\}\right).
\end{align}
\end{corollary}

In summary, when the batch size is large enough, DAT converges to a first-order stationary point of problem \eqref{eq: prob_DAT} and there is a linear speed-up in terms of $M$ with respect to $\sigma^2$.  Next, we provide the details of the proof.

\newpage 

\section{Proof Details}
\label{app: analysis}

\subsection{Preliminaries}
In the proof, we use the following inequality and notations.

1. Young’s inequality with parameter $\epsilon$ is
\begin{equation}
\langle \bx,\by\rangle\le\frac{1}{2\epsilon}\|\bx\|^2 + \frac{\epsilon}{2}\|\by\|^2,
\end{equation}
where $\bx,\by$ are two vectors.

2. Define the historical trajectory of the iterates as $\mathcal{F}_t=\{\btheta_{t-1},\ldots,\btheta_1\}$.

3. We denote vector $[\bx]_i$ as the parameters at the $i$th layer of the neural net and $[\bx]_{ij}$ represents the $j$th entry of the parameter at the $i$th layer.

4. We define
\begin{equation}\label{eq.qeq}
    \bgo:= \frac{1}{M}\sum^M_{i=1} \mathbb{E}_{\bx_t\in\mathcal{B}^{(i)}}\left(\lambda\nabla l(\btheta_t;\bx_t)+\nabla_{\btheta}\phi(\btheta_t,\bdelta^{(i)}_t(\bx_t);\bx_t)\right)=\frac{1}{M} \sum_{i=1}^M \mathbf g_t^{(i)}.
\end{equation}

\subsection{Details of LAMB algorithm}
\begin{algorithm}[H]
\caption{{LAMB \citep{you2019large}}}
\label{alg:p1}
\begin{algorithmic}
\State Input: learning rate $\eta_t$, $0<\beta_1,\beta_2<1$, scaling function $\layerscale(\cdot)$, $\zeta>0$
\For{$t=1,\ldots$}
\State $\bm_t=\beta_1\bm_{t-1}+(1-\beta_1)\bg$, {where $\hat{\mathbf g}_t$ is given by \eqref{eq: grad_agg}}
\State $\bv_t=\beta_2\bv_{t-1}+(1-\beta_2)\bg^2$
\State $\bm_t=\bm_t/(1-\beta^t_1)$
\State $\bv_t=\bv_t/(1-\beta^t_2)$
\State Compute ratio $\bu_t=\frac{\bm_t}{\sqrt{\bv_t}+\zeta}$
\EndFor
\State Update
\begin{equation}
\btheta_{t+1,i}=\btheta_{t,i}-\frac{\eta_t\layerscale(\|\btheta_{t,i}\|)}{\|\bu_{t,i}\|}\bu_{t,i}.\label{eq.upth}
\end{equation}
\end{algorithmic}
\end{algorithm}

\subsection{Proof of \leref{le.vererror}}
\begin{proof}
From A2, we have
\begin{equation}
    \left\|\nabla \phi\left(\btheta_t,\bdelta^{(i)}_t(\bx_t);\bx_t\right)-\nabla \phi\left(\btheta_t,(\bdelta^*)^{(i)}_t(\bx_t);\bx_t\right)\right\|\le L_{\phi}\|\bdelta^{(i)}_t(\bx_t)-(\bdelta^*)^{(i)}_t(\bx_t)\|.\label{eq.lipphi}
\end{equation}

Also, we know that function $\phi(\btheta,\bdelta,\bx)$ is strongly concave with respect to $\bdelta$, so we have
\begin{multline}
\mu\|\bdelta^{(i)}_t(\bx_t)-(\bdelta^*)^{(i)}_t(\bx_t)\|
\\
\le\left\langle\nabla_{\bdelta}\phi(\btheta_t,(\bdelta^*)^{(i)}_t(\bx_t);\bx_t)-\nabla_{\bdelta}\phi(\btheta_t,\bdelta^{(i)}_t(\bx_t);\bx_t),\bdelta^{(i)}_t(\bx_t)- (\bdelta^*)^{(i)}_t(\bx_t)\right\rangle.\label{eq.stonrc}
\end{multline}

Next, we have two conditions about the qualities of solutions $\bdelta^{(i)}_t(\bx_t)$ and $(\bdelta^*)^{(i)}_t(\bx_t)$.
First, we know that $\bdelta^{(i)}_t(\bx_t)$ is a-$\varepsilon$ approximate solution to $(\bdelta^*)^{(i)}_t(\bx_t)$, so we have
\begin{equation}
    \left\langle(\bdelta^*)^{(i)}_t(\bx_t)-\bdelta^{(i)}_t(\bx_t),\nabla_{\bdelta}\phi(\btheta_t,\bdelta^{(i)}_t(\bx_t);\bx_t)\right\rangle\le\varepsilon.
\end{equation}
Second, since $(\bdelta^*)^{(i)}_t(\bx_t)$ is the optimal solution, it satisfies
\begin{equation}
    \left\langle(\bdelta^{(i)}_t(\bx_t)-(\bdelta^*)^{(i)}_t(\bx_t),\nabla_{\bdelta}\phi(\btheta_t,(\bdelta^*)^{(i)}_t(\bx_t);\bx_t)\right\rangle\le0.
\end{equation}
Adding them together, we can obtain
\begin{equation}
    \left\langle\bdelta^{(i)}_t(\bx_t)-(\bdelta^*)^{(i)}_t(\bx_t), \nabla_{\bdelta}\phi(\btheta_t,(\bdelta^*)^{(i)}_t(\bx_t);\bx_t)-\nabla_{\bdelta}\phi(\btheta_t,\bdelta^{(i)}_t(\bx_t);\bx_t)\right\rangle\le\varepsilon.\label{eq.optc}
\end{equation}
Substituting \eqref{eq.optc} into \eqref{eq.stonrc}, we can get
\begin{equation}
    \mu\|\bdelta^{(i)}_t(\bx_t)-(\bdelta^*)^{(i)}_t(\bx_t)\|^2\le\varepsilon.
\end{equation}

Combining \eqref{eq.lipphi}, we have
\begin{equation}
    \left\|\nabla \phi(\btheta_t,\bdelta^{(i)}_t(\bx_t);\bx_t)-\nabla \phi(\btheta_t,(\bdelta^*)^{(i)}_t(\bx_t);\bx_t)\right\|^2\le L^2_{\phi}\frac{\varepsilon}{\mu}.
\end{equation}
\end{proof}

\subsection{Descent of quantized LAMB}
First, we provide the following lemma as a stepping stone for the subsequent analysis.
\begin{lemma}\label{le.desc}
Under A1--A3,  suppose that sequence $\{\btheta_t\}$ is generated by DAT. Then, we have
\begin{equation}
\mathbb{E}[-\langle\nabla \Psi(\btheta_t),\bg\rangle] \le-\frac{\mathbb{E}\|\nabla\Psi(\btheta_t)\|^2}{2}+\varepsilon+\frac{(1+\lambda)\sigma^2}{MB}\label{eq.deskey}.
\end{equation}
\end{lemma}
\begin{proof}
From \eqref{eq.optdelta}, \eqref{eq.defPhi} and A2, we know that 
\begin{equation}
    \nabla_{\btheta} \Phi_i(\btheta,\bx) = \nabla_{\btheta} \phi(\btheta,(\bdelta^*)^{(i)}(\bx);\bx),
\end{equation}
so we can get
\begin{align}
  \nabla_{\btheta}\Psi(\btheta)=&\frac{1}{M}\sum^M_{i=1}\lambda \nabla_{\btheta}l_i(\btheta)+\nabla_{\btheta}\Phi_i(\btheta) 
  \\
  =& \lambda\nabla_{\btheta} l(\btheta)+\frac{1}{M}\sum^M_{i=1} \mathbb{E}_{\bx\in\mathcal{D}^{(i)}}\nabla_{\btheta}\phi(\btheta,(\bdelta^*)^{(i)}(\bx);\bx)
  \\
  :=& \bar{\mathbf{g}}(\btheta).\label{eq.unb}
\end{align}

Then, we have
\begin{align}
    \mathbb{E}\langle \nabla \Psi(\btheta_t),\bgo\rangle=&\mathbb{E}\langle \nabla \Psi(\btheta_t),\bar{\mathbf{g}}_t\rangle+\mathbb{E}\langle \nabla \Psi(\btheta_t), \bgo-\bar{\mathbf{g}}_t\rangle
    \\
    =&\mathbb{E}_{\mathcal{F}_t}\mathbb{E}_{\bx_t|\mathcal{F}_t}\langle \nabla \Psi(\btheta_t),\bar{\mathbf{g}}_t\rangle+\mathbb{E}\langle \nabla \Psi(\btheta_t), \bgo-\bar{\mathbf{g}}_t\rangle
    \\
    \mathop{=}\limits^{\eqref{eq.unb}}&\mathbb{E}\|\nabla \Psi(\btheta_t)\|^2+\mathbb{E}\langle \nabla \Psi(\btheta_t), \bgo-\bar{\mathbf{g}}_t\rangle
     \\
    =&\mathbb{E}\|\nabla \Psi(\btheta_t)\|^2+\mathbb{E}\langle \nabla \Psi(\btheta_t),\bgo-\bgo^*\rangle+\mathbb{E}\langle \nabla \Psi(\btheta_t), \bgo^*-\bar{\mathbf{g}}_t\rangle
\end{align}
where
\begin{equation}
    \bar{\mathbf{g}}_t:=\frac{1}{M}\sum^M_{i=1} \mathbb{E}_{\bx_t\in\mathcal{D}^{(i)}}\left(\lambda\nabla l(\btheta_t,\bx_t)+\nabla_{\btheta}\phi(\btheta_t,(\bdelta^*)^{(i)}_t(\bx_t);\bx_t)\right)=\lambda\nabla l(\btheta_t)+\nabla \Phi(\btheta_t),\label{eq.defbgt}
\end{equation}
and 
\begin{equation}
    \bgo^*:=\frac{1}{M}\sum^M_{i=1} \mathbb{E}_{\bx_t\in \mathcal{B}^{(i)}}\left(\lambda\nabla l(\btheta_t,\bx_t)+\nabla_{\btheta}\phi(\btheta_t,(\bdelta^*)^{(i)}_t(\bx_t);\bx_t)\right).\label{eq.defbgtp}
\end{equation}

Next, we can quantify the different between $\bgo$ and $\bgo^*$ by gradient Lipschitz continuity of function $\layerscale(\cdot)$ as the following
\begin{equation}
\mathbb{E}\|\bgo-\bgo^*\|^2
\mathop{\le}\limits^{(a)}  \frac{1}{M}\sum^M_{i=1}\mathbb{E}_{\mathcal{F}_t}\mathbb{E}_{\bx_t|\mathcal{F}_t}\left[\|\nabla_{\btheta} \phi(\btheta_t,(\bdelta^*)^{(i)}(\bx_t);\bx_t)- \nabla_{\btheta} \phi(\btheta_t,\bdelta^{(i)}(\bx_t);\bx_t)\|^2\right]
\mathop{\le}\limits^{\eqref{eq.maxorl}} \varepsilon\label{eq.ue}
\end{equation}
where in $(a)$ we use Jensen's inequality.

And the difference between $\bar{\mathbf{g}}_t$ and $\bgo^*$ can be upper bounded by 
\begin{align}\notag
    \mathbb{E}\|\bar{\mathbf{g}}_t-\bgo^*\|^2=&\mathbb{E}_{\mathcal{F}_t}\left\|\frac{1}{M}\sum^M_{i=1}\mathbb{E}_{\bx_t|\mathcal{F}_t}\nabla_{\btheta} \phi(\btheta_t,(\bdelta^*)^{(i)}(\bx_t);\bx_t)- \nabla_{\btheta} \phi(\btheta_t)\right\|^2
\\
&+\lambda \mathbb{E}_{\mathcal{F}_t}\left\|\frac{1}{M}\sum^M_{i=1}\mathbb{E}_{\bx_t|\mathcal{F}_t}\nabla l(\btheta_t;\bx_t)-\nabla l(\btheta_t)\right\|^2
\\
\mathop{=}\limits^{A3} &\frac{(1+\lambda)\sigma^2}{MB}.\label{eq.bdggp}
\end{align}

Applying Young’s inequality with parameter 2, we have
\begin{align}
\mathbb{E}[-\langle\nabla \Psi(\btheta_t),\bgo\rangle]\le & -\mathbb{E}\|\nabla\Psi(\btheta_t)\|^2+\frac{\mathbb{E}\|\nabla\Psi(\btheta_t)\|^2}{2}+\mathbb{E}\|\bar{\mathbf{g}}_t-\bgo^*\|^2+\mathbb{E}\|\bgo^*-\bgo\|^2
\\
\mathop{\le}\limits^{\eqref{eq.ue}}&-\frac{\mathbb{E}\|\nabla\Psi(\btheta_t)\|^2}{2}+\varepsilon+\frac{(1+\lambda)\sigma^2}{MB}.
\end{align}

\end{proof}

\subsection{Proof of \thref{th:main}}
\begin{proof}
We set $\beta_1=0$ in LAMB for simplicity. From gradient Lipschitz continuity, we have
\begin{align}
\Psi(\btheta_{t+1})\mathop{\le}\limits^{A1} &\Psi(\btheta_t)+\sum^h_{i=1}\langle[\nabla_{\btheta} \Psi(\btheta_{t})]_i,\btheta_{t+1,i}-\btheta_{t,i}\rangle+\sum^h_{i=1}\frac{L_i}{2}\|\btheta_{t+1,i}-\btheta_{t,i}\|^2
\\
\mathop{\le}\limits^{(a)} &\Psi(\btheta_t)\underbrace{-\eta_t\sum^h_{i=1}\sum^{d_i}_{j=1}\layerscale(\|\btheta_{t,i}\|)\left\langle[\nabla \Psi(\btheta_{t})]_{ij},\frac{[\bu_{t}]_{ij}}{\|\bu_{t,i}\|}\right\rangle}_{:=\boldsymbol{\mathcal{R}}}+\sum^h_{i=1}\frac{\eta^2_t c^2_u L_i}{2},\label{eq.Phibd}
\end{align}
where in $(a)$ we use \eqref{eq.upth}, and the upper bound of $\layerscale(\|\btheta_{t,i}\|)$.

Next, we split term $R$ as two parts by leveraging $\textrm{sign}([\nabla\Psi(\btheta_t)]_{ij})$ and $\textrm{sign}([\bu_{t}]_{ij})$ as follows.
\begin{align}\notag
\boldsymbol{\mathcal{R}}=&-\eta_t\sum^h_{i=1}\sum^{d_i}_{j=1}\layerscale(\|\btheta_{t,i}\|)[\nabla \Psi(\btheta_t)]_{ij}\frac{[\bu_{t}]_{ij}}{\|\bu_{t,i}\|}\mathbbm{1}\left(\textrm{sign}([\nabla \Psi(\btheta_t)]_{ij})=\textrm{sign}([\bu_{t}]_{ij})\right)
\\
&-\eta_t\sum^h_{i=1}\sum^{d_i}_{j=1}\layerscale(\|\btheta_{t,i}\|)[\nabla \Psi(\btheta_t)]_{ij}\frac{[\bu_{t}]_{ij}}{\|\bu_{t,i}\|}\mathbbm{1}\left(\textrm{sign}([\nabla \Psi(\btheta_t)]_{ij})\neq\textrm{sign}([\bu_{t}]_{ij})\right)
\\\notag
\mathop{\le}\limits^{(a)}&-\eta_t c_l\sum^h_{i=1}\sum^{d_i}_{j=1}\sqrt{\frac{1-\beta_2}{G^2d_i}}[\nabla \Psi(\btheta_t)]_{ij} [\bg]_{ij}\mathbbm{1}\left(\textrm{sign}([\nabla [\Psi(\btheta_t)]_{ij})=\textrm{sign}([\bg]_{ij})\right)
\\
&-\eta_t\sum^h_{i=1}\sum^{d_i}_{j=1}\layerscale(\|\btheta_{t,i}\|)[\nabla \Psi(\btheta_t)]_{ij}\frac{[\bu_{t}]_{ij}}{\|\bu_{t,i}\|}\mathbbm{1}\left(\textrm{sign}([\nabla \Psi(\btheta_t)]_{ij})\neq\textrm{sign}([\bu_t]_{ij})\right)
\\\notag
\mathop{\le}\limits^{(b)}&-\eta_tc_l\sum^h_{i=1}\sum^{d_i}_{j=1}\sqrt{\frac{1-\beta_2}{G^2d_i}}[\nabla \Psi(\btheta_t)]_{ij}  [\bg]_{ij}
\\
&-\eta_t\sum^h_{i=1}\sum^{d_i}_{j=1}\layerscale(\|\btheta_{t,i}\|)[\nabla \Psi(\btheta_t)]_{ij}\frac{[\bu_{t}]_{ij}}{\|\bu_{t,i}\|}\mathbbm{1}\left(\textrm{sign}([\nabla \Psi(\btheta_t)]_{ij})\neq\textrm{sign}([\bu_{t}]_{ij})\right).\label{eq.t1bd}
\end{align}
where in $(a)$ we use the fact that $\|\bu_{t,i}\|\le \sqrt{\frac{d_i}{1-\beta_2}}$ and $\sqrt{\bv_t}\le G$, and in $(b)$ we add 
\begin{equation}
-\eta_t c_l\sum^h_{i=1}\sum^{d_i}_{j=1}\sqrt{\frac{1-\beta_2}{G^2d_i}}[\nabla \Psi(\btheta_t)]_{ij}[\bg]_{ij}\mathbbm{1}\left(\textrm{sign}([\nabla \Psi(\btheta_t)]_{ij})\neq\textrm{sign}([\bg]_{ij})\right)\ge0.
\end{equation}

Taking expectation on both sides of \eqref{eq.t1bd}, we have the following:
\begin{align}\notag
\mathbb{E}[\boldsymbol{\mathcal{R}}]\le&\underbrace{-\eta_t c_l\sqrt{\frac{h(1-\beta_2)}{G^2d}}\sum^h_{i=1}\sum^{d_i}_{j=1}\mathbb{E}[[\nabla \Psi(\btheta_t)]_{ij}[\bg]_{ij}}_{:=\boldsymbol{\mathcal{U}}}
\\
&+\underbrace{\eta_t c_u\sum^h_{i=1}\sum^{d_i}_{j=1}\mathbb{E}\left[[\nabla \Psi(\btheta_t)]_{ij}\mathbbm{1}\left(\textrm{sign}([\nabla \Psi(\btheta_t)]_{ij})\neq\textrm{sign}([\bu_t]_{ij})\right)\right]}_{:=\boldsymbol{\mathcal{V}}}.
\end{align}

Next, we will get the upper bounds of $\boldsymbol{\mathcal{U}}$ and $\boldsymbol{\mathcal{V}}$ separably as follows. First, we write the inner product between $[\nabla \Psi(\btheta)]_{ij}$ and $[\bg]_{ij}$ more compactly,
\begin{align}
\boldsymbol{\mathcal{U}}\le&-\eta_t c_l\sqrt{\frac{h(1-\beta_2)}{G^2d}}\sum^h_{i=1}\mathbb{E}\left\langle [\nabla \Psi(\btheta)]_i, [\bg]_i\right\rangle
\\
\le&-\eta_t c_l\sqrt{\frac{h(1-\beta_2)}{G^2d}}\sum^h_{i=1}\mathbb{E}\left\langle [\nabla \Psi(\btheta_t)]_i,[\bg]_i-[\bgo]_i+[\bgo]_i\right\rangle
\\
\le&-\eta_t c_l\sqrt{\frac{h(1-\beta_2)}{G^2d}}
\left(\mathbb{E}\left\langle\nabla\Psi(\btheta),\bgo\right\rangle+\sum^h_{i=1}\mathbb{E}\left\langle [\nabla \Psi(\btheta_t)]_i,[\bg]_i-[\bgo]_i\right\rangle\right).
\end{align}

Applying \leref{le.desc}, we can get
\begin{align}\notag
\boldsymbol{\mathcal{U}}\mathop{\le}\limits^{\eqref{eq.deskey}} &-\eta_t c_l\sqrt{\frac{h(1-\beta_2)}{G^2d}}\frac{1}{2}\mathbb{E}\|\nabla \Psi(\btheta_t)\|^2
+\eta_t c_l\sqrt{\frac{h(1-\beta_2)}{G^2d}}\left(\varepsilon+\frac{(1+\lambda)\sigma^2}{MB}\right)
\\
&-\eta_t c_l\sqrt{\frac{h(1-\beta_2)}{G^2d}}\sum^h_{i=1}\mathbb{E}\left\langle [\nabla \Psi(\btheta_t)]_i,[\bg]_i -[\bgo]_i\right\rangle
\\\notag
\mathop{\le}\limits^{(a)}&-\eta_t c_l\sqrt{\frac{h(1-\beta_2)}{G^2d}}\frac{1}{2}\mathbb{E}\|\nabla \Psi(\btheta_t)\|^2
+\eta_t c_l\sqrt{\frac{h(1-\beta_2)}{G^2d}}\left(\varepsilon+\frac{(1+\lambda)\sigma^2}{MB}\right)
\\
&+\frac{\eta_t c_l}{4}\sqrt{\frac{h(1-\beta_2)}{G^2d}}\mathbb{E}\| \nabla \Psi(\btheta_t)\|^2
+c_l\eta_t\sqrt{\frac{h(1-\beta_2)}{G^2d}}\mathbb{E}\|\bg-\bgo\|^2
\\\notag
\mathop{\le}\limits^{(b)} &-\frac{\eta_t c_l}{4}\sqrt{\frac{h(1-\beta_2)}{G^2d}}\frac{1}{2}\mathbb{E}\|\nabla \Psi(\btheta_t)\|^2+\eta_t c_l\sqrt{\frac{h(1-\beta_2)}{G^2d}}\left(\varepsilon+\frac{(1+\lambda)\sigma^2}{MB}\right)
\\
&+\eta_t c_l\sqrt{\frac{h(1-\beta_2)}{G^2d}}\delta^2\label{eq.ubd}
\end{align}
where we use the in $(a)$ we use Young's inequality (with parameter $2$),  and in $(b)$ we have
\begin{equation}
    \mathbb{E}\|\bg-\bgo\|^2=\mathbb{E}\left\|\frac{1}{M}\sum^M_{i=1}Q(\bgo^{(i)})-\bgo^{(i)}\right\|^2\mathop{\le}\limits^{A4}\delta^2.\label{eq.bdq}
\end{equation}

Second, we give the upper of $\boldsymbol{\mathcal{V}}$:
\begin{align}\label{eq.vbd}
\boldsymbol{\mathcal{V}}\le&\eta_t c_u\sum^h_{i=1}\sum^{d_i}_{j=1}[\nabla \Psi(\btheta_t)]_{ij}\underbrace{\mathbbm{P}\left(\textrm{sign}([\nabla \Psi(\btheta_t)]_{ij})\neq\textrm{sign}([\bg]_{ij})\right)}_{:=\boldsymbol{\mathcal{W}}}
\end{align}
where the upper bound of $\boldsymbol{\mathcal{W}}$ can be quantified by using Markov’s inequality followed by Jensen’s inequality as the following:
\begin{align}
\notag
\boldsymbol{\mathcal{W}}=&\mathbbm{P}\left(\textrm{sign}([\nabla \Psi(\btheta_t)]_{ij})\neq\textrm{sign}([\bg]_{ij})\right)
\\
\le&\mathbbm{P}[|[\nabla \Psi(\btheta_t)]_{ij}-[\bg]_{ij}|>[\nabla \Psi(\btheta_t)]_{ij}]
\\
\le&\frac{\mathbb{E}[[\nabla \Psi(\btheta_t)]_{ij}-[\bg]_{ij}]}{|[\nabla \Psi(\btheta_t)]_{ij}|}
\\
\le&\frac{\sqrt{\mathbb{E}[([\nabla \Psi(\btheta_t)]_{ij}-[\bg]_{ij})^2]}}{|[\nabla \Psi(\btheta_t)]_{ij}|}
\\
\mathop{\le}\limits^{\eqref{eq.unb}}&\frac{\sqrt{\mathbb{E}[([\bar{\mathbf{g}}_t]_{ij}-[\bgo^*]_{ij} + [\bgo^*]_{ij}- [\bgo]_{ij}+[\bgo]_{ij}-[\bg]_{ij})^2]}}{|[\nabla \Psi(\btheta_t)]_{ij}|}
\\
\mathop{\le}\limits^{(a)}&\sqrt{3}\frac{\sqrt{\frac{(1+\lambda)\sigma^2_{ij}}{M|\mathcal{B}|}+\epsilon_{ij}+\delta^2_{ij}}}{|[\nabla \Psi(\btheta_t)]_{ij}|} \label{eq.bdw}
\end{align}
where $(a)$ is true due to the following relations:
\emph{i}) from \eqref{eq.bdggp}, we have
\begin{equation}\label{eq.key}
 \mathbb{E}[([\bar{\mathbf{g}}_t]_{ij}-[\bgo^*]_{ij})^2] \le\frac{(1+\lambda)\sigma^2_{ij}}{MB};
\end{equation}
\emph{ii}) from \eqref{eq.ue}, we can get
\begin{equation}
  \mathbb{E}[([\bgo]_{ij}-[\bgo^*]_{ij})^2]\le \varepsilon_{ij};
\end{equation}
and \emph{iii}) from \eqref{eq.bdq}, we know
\begin{equation}
    \mathbb{E}[([\bg]_{ij}-[\bgo]_{ij})^2]\le\delta^2_{ij}.
\end{equation}

Therefore, combining \eqref{eq.Phibd} with the upper bound of $\boldsymbol{\mathcal{U}}$ shown in \eqref{eq.ubd} and $\boldsymbol{\mathcal{V}}$ shown in \eqref{eq.vbd}\eqref{eq.bdw}, we have
\begin{align}\notag
\mathbb{E}[\Psi(\btheta_{t+1})]\le& \mathbb{E}[\Psi(\btheta_t)]-\eta_t c_l\sqrt{\frac{h(1-\beta_2)}{G^2d}}\frac{1}{4}\mathbb{E}\|\nabla \Psi(\btheta_t)\|^2 +\eta_t c_l\sqrt{\frac{h(1-\beta_2)}{G^2d}}\left(\varepsilon+\frac{(1+\lambda)\sigma^2}{MB}\right)
\\
&+\eta_t c_l\sqrt{\frac{h(1-\beta_2)}{G^2d}}\delta^2+\eta_t c_u\sqrt{3}\sum^h_{i=1}\sum^{d_i}_{j=1}\sqrt{\frac{(1+\lambda)\sigma^2_{ij}}{MB}+\varepsilon_{ij}+\delta^2_{ij}}+\frac{\eta^2_t c^2_u \sum^h_{i=1}L_i}{2}.
\end{align}

Note that the error vector $\boldsymbol{\chi}$ is defined as the following
\begin{equation}
    \boldsymbol{\chi}=\left[\begin{matrix}\sqrt{\frac{(1+\lambda)\sigma^2_{11}}{M|\mathcal{B}|}+\varepsilon_{11}+\delta^2_{11}} \\
    \vdots \\
    \sqrt{\frac{(1+\lambda)\sigma^2_{ij}}{M|\mathcal{B}|}+\varepsilon_{ij}+\delta^2_{ij}}
    \\
    \vdots \\
    \sqrt{\frac{(1+\lambda)\sigma^2_{hd_h}}{M|\mathcal{B}|}+\varepsilon_{hd_h}+\delta^2_{hd_h}}
    \end{matrix}\right]\in\mathbb{R}^{d},
\end{equation}
and we have
\begin{equation}
    L=\left[\begin{matrix}L_1\\ \vdots \\ L_h\end{matrix}\right]\in\mathbb{R}^h.
\end{equation}

Recall
\begin{equation}
\kappa=\frac{c_u}{c_l}.
\end{equation}
Rearranging the terms, we can arrive at
\begin{align}\notag
\underbrace{\sqrt{\frac{h(1-\beta_2)}{G^2d}}\frac{1}{4}}_{:=C}\left(\|\nabla \Psi(\btheta_t)\|^2\right)\le & \frac{\mathbb{E}[\Psi(\btheta_t)]-\mathbb{E}[\Psi(\btheta_{t+1})]}{\eta_t c_l}+ 4C\delta^2+ 2C\left(\varepsilon+\frac{(1+\lambda)\sigma^2}{MB}\right)
\\
&+\sqrt{3}\kappa\|\boldsymbol{\chi}\|_1+\frac{\eta_t c_u\kappa \|L\|_1}{2}.
\end{align}

Applying the telescoping  sum over $t=1,\ldots,T$, we have
\begin{align}\notag
\frac{1}{T}\sum^{\top}_{t=1}\mathbb{E}\|\nabla_{\btheta} \Psi(\btheta_t)\|^2\le&\frac{\mathbb{E}[\Psi(\btheta_1)]-\mathbb{E}[\Psi(\btheta_{T+1})]}{\eta_t c_l CT}+2\left(\varepsilon+\frac{(1+\lambda)\sigma^2}{MB}\right)+4\delta^2
\\
&+\frac{\kappa\sqrt{3}}{C}\|\boldsymbol{\chi}\|_1+\frac{\eta_t c_u\kappa \|L\|_1}{2C}.
\end{align}
\end{proof}

\newpage
\clearpage

\section{Additional Experiments}\label{app: sect}

\subsection{Training details}
\label{app: train_setting}

ImageNet AT and Fast AT experiments are conducted at a single computing node with dual 22-core CPU, 512GB RAM and 6 Nvidia V100 GPUs. The training epoch is $30$ by calling for the momentum SGD optimizer. The weight decay and momentum parameters are set to $0.0001$ and $0.9$. The initial learning rate is set to $0.1$ (tuned over $\{0.01, 0.05, 0.1, 0.2\}$), which is
decayed by $\times 1/10$ at the training epoch $20, 25, 28$, respectively.

ImageNet DAT experiments are conducted at 
$\{ 1,3,6\}$
computing nodes with dual 22-core CPU, 512GB RAM and 6 Nvidia V100 GPUs. The training epoch is $30$ by calling for the  LAMB optimizer. The weight decay is set to $0.0001$. $\beta_1$ and $\beta_2$ are set to $0.9$ and $0.999$. The initial learning rate $\eta_1$  is tuned over \{$0.01$, $0.05$, $0.1$, $0.2$, $0.4$\}, which is decayed by $\times 1/10$ at the training epoch $20, 25, 28$, respectively. To execute algorithms with the initial learning rate $\eta_1$ greater than $0.2$, we choose the   model weights after $5$-epoch warm-up as its initialization for DAT, where each warm-up epoch $k$ uses the linearly increased learning rate $(k/5)\eta_1$. 

\subsection{Additional results}
\label{app: add_results}

\paragraph{Discussion on cyclic learning rate.}
 
It was shown in \citep{Wong2020Fast} that the use of a cyclic learning rate (CLR) trick can further accelerate  the Fast  AT algorithm in the small-batch   setting \citep{Wong2020Fast}. In Figure\,\ref{fig: cyclic_lr_batch_size}, we present the performance of Fast AT with CLR versus batch sizes.
We observe that when CLR meets the large-batch setting, it becomes significantly worse than its performance in the small-batch setting. The   reason is that  CLR requires a certain number of iterations
to proceed with the cyclic schedule. However, the use of large data batch  only results in  a small amount of iterations by fixing   the number of epochs. 

 \begin{figure}[htb]
    \vspace*{-0.0in}
\centerline{
\begin{tabular}{c}
\includegraphics[width=.5\textwidth,height=!]{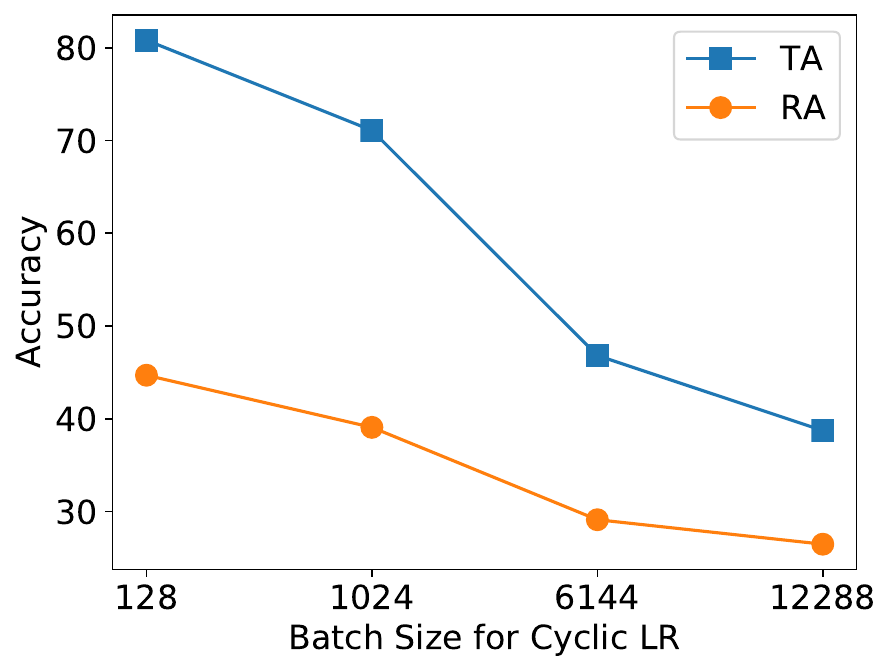} 
\end{tabular}}
\caption{\footnotesize{
TA/RA of Fast AT with CLR versus batch sizes on (CIFAR-10, ResNet-18).
}}
  \label{fig: cyclic_lr_batch_size}
  \vspace*{-0.00in}
\end{figure}

\paragraph{Additional details on HPC setups.}
To further reduce communication cost, we also conduct DAT at a HPC cluster. The computing nodes of the cluster are connected with InfiniBand (IB) and PCIe Gen4 switch. To compare with results in Table\,\ref{table: overall}, we use 6 of 57 nodes of the cluster. Each node has 6 Nvidia V100s which are interconnected with NVLink. We use Nvida NCCL as communication backend. In Table\,\ref{table: quadtization_imagenet}, we have presented the performance of DAT for ImageNet, ResNet-50 with use of HPC compared to standard (non-HPC) distributed system.  

\newpage
\clearpage

\end{document}